\documentclass{article}

\usepackage{microtype}
\usepackage{graphicx}
\usepackage{subfigure}
\usepackage{booktabs} 

\usepackage{hyperref}

\usepackage[accepted]{icml2024}

\usepackage[utf8]{inputenc} 
\usepackage[T1]{fontenc}    
\usepackage{url}            
\usepackage{booktabs}       
\usepackage{amsfonts}       
\usepackage{nicefrac}       
\usepackage{microtype}      
\usepackage{xcolor}         
\usepackage{dblfloatfix}    

\usepackage{amsmath,bm}
\usepackage{amsthm}
\usepackage{amssymb}
\usepackage{mathtools}
\usepackage{centernot}
\usepackage{algorithm}
\usepackage{algorithmic}
\usepackage{comment}

\usepackage{chngcntr}
\usepackage{apptools}
\AtAppendix{\counterwithin{lemma}{section}}
\AtAppendix{\counterwithin{corollary}{section}}
\AtAppendix{\counterwithin{theorem}{section}}
\AtAppendix{\counterwithin{proposition}{section}}
\AtAppendix{\counterwithin{definition}{section}}
\AtAppendix{\counterwithin{assumption}{section}}
\AtAppendix{\counterwithin{example}{section}}

\usepackage{tikz}
\usepackage{tikz-cd}

\usepackage[
  open,
  openlevel=4,
  depth=4,
  atend,
  numbered
]{bookmark}

\DeclareMathOperator*{\argmax}{arg\,max}

\newcommand{\E}{\operatorname{\mathbb E}}
\newcommand{\innermid}{\;\middle\lvert\;}

\newtheorem{lemma}{Lemma}

\newtheorem{remark}{Remark}
\newtheorem{corollary}{Corollary}
\newtheorem{theorem}{Theorem}
\newtheorem{proposition}{Proposition}

\newtheorem{assumption}{Assumption}

\newenvironment{subproof}[1][\proofname]{%
  \begin{proof}[#1]%
}{%
  \end{proof}%
}

\newtheorem{altassumption}{Assumption}[assumption]
\newenvironment{assumption+}[1]
  {\renewcommand{\thealtassumption}{\ref*{#1}$'$}%
   \begin{altassumption}}
  {\end{altassumption}}

\newtheorem{subassumption}{Assumption}[assumption]

\newenvironment{assumption*}
 {\ifnum \value{subassumption}=0
 \stepcounter{assumption}\fi\subassumption}
 {\endsubassumption}
 
\usepackage{pifont}
\newcommand{\cmark}{\ding{51}}%
\newcommand{\xmark}{\ding{55}}%

\allowdisplaybreaks

\icmltitlerunning{Major-Minor Mean Field Multi-Agent Reinforcement Learning}

\begin{document}

\twocolumn[
\icmltitle{Major-Minor Mean Field Multi-Agent Reinforcement Learning}
           



\begin{icmlauthorlist}
\icmlauthor{Kai Cui}{yyy}
\icmlauthor{Christian Fabian}{yyy}
\icmlauthor{Anam Tahir}{yyy}
\icmlauthor{Heinz Koeppl}{yyy}
\end{icmlauthorlist}

\icmlaffiliation{yyy}{Department of Electrical Engineering and Information Technology, Technische Universität Darmstadt, Darmstadt, Germany}

\icmlcorrespondingauthor{Kai Cui}{kai.cui@tu-darmstadt.de}
\icmlcorrespondingauthor{Heinz Koeppl}{heinz.koeppl@tu-darmstadt.de}

\icmlkeywords{Multi-Agent Reinforcement Learning, Mean Field Control, Large-Scale Multi-Agent Systems}

\vskip 0.3in
]



\printAffiliationsAndNotice{}  

\begin{abstract}
Multi-agent reinforcement learning (MARL) remains difficult to scale to many agents. Recent MARL using Mean Field Control (MFC) provides a tractable and rigorous approach to otherwise difficult cooperative MARL. However, the strict MFC assumption of many independent, weakly-interacting agents is too inflexible in practice. We generalize MFC to instead simultaneously model many similar and few complex agents -- as Major-Minor Mean Field Control (M3FC). Theoretically, we give approximation results for finite agent control, and verify the sufficiency of stationary policies for optimality together with a dynamic programming principle. Algorithmically, we propose Major-Minor Mean Field MARL (M3FMARL) for finite agent systems instead of the limiting system. The algorithm is shown to approximate the policy gradient of the underlying M3FC MDP. Finally, we demonstrate its capabilities experimentally in various scenarios. We observe a strong performance in comparison to state-of-the-art policy gradient MARL methods. 
\end{abstract}

\section{Introduction}
Recent successes of reinforcement learning (RL) \citep{vinyals2019grandmaster, schrittwieser2020mastering, ouyang2022training} motivate the search for techniques for the multi-agent case, referred to as multi-agent reinforcement learning (MARL). Due to the high complexity of multi-agent control \citep{bernstein2002complexity, daskalakis2009complexity}, exploiting problem structure is important for scalable MARL. In this work, we consider systems with many agents interacting through aggregated information of all agents -- the mean field (MF).

\paragraph{Mean field control for MARL.}
Dynamical control and behavior in systems with many agents is the subject of studies in mean field games (MFG) \citep{huang2006large, lasry2007mean} and mean field control (MFC) \citep{nourian2012nash, bensoussan2013mean, carmona2019model}. Such aggregated interaction models simplify MARL in the limit of infinite agents, whenever agents interact only through their empirical distribution. The simplification provides a problem complexity that is independent of the exact number of agents. The result is tractability, by avoiding otherwise exponentially large joint state-action spaces \citep{zhang2021multi}. This has led to scalable control based on MFC \citep{gu2019dynamic, carmona2019model}. And indeed, in applications such aggregation is commonly found on some level, e.g., in chemical reaction networks for aggregate molecule mass \citep{anderson2011continuous}, related mass-action epidemics models \citep{kiss2017mathematics}, or traffic where congestion depends on the number of travelling cars \citep{cabannes2021solving}, to name just a few. See also epidemics control \citep{dunyak2021large}, drone swarms \citep{shiri2019massive}, self organization \citep{carmona2022synchronization}, and many more financial \citep{carmona2020applications} or engineering scenarios \citep{djehiche2017mean}.

\begin{table*}[b]
    \centering
    \caption{A comparison of recent related works and a subset of their results on \textit{discrete-time} MFC. \\\textit{prop. chaos}: propagation of chaos; \textit{opt. policy}: existence of optimal (stationary) policies; \textit{common noise}: presence thereof; \textit{non-finite}: non-finite state-actions, e.g. compact; \textit{major agent}: presence thereof; \textit{RL}: RL algorithm ($^+$: learns / is analyzed on finite MARL problems).}
    \vspace{0.3cm}
    \label{tab:diff}
    \renewcommand{\arraystretch}{1.21}
    \begin{tabular}{cccccccc}
        \toprule 
        Ref. & \textit{prop. chaos} & \textit{opt. policy} & \textit{common noise} & \textit{non-finite} & \textit{major agent} & \textit{RL} \\ \midrule
        \citet{carmona2019model} & \xmark & \cmark & \cmark & \cmark & \xmark & \cmark   \\
        \citet{gu2021mean, gu2019dynamic} & \cmark & \cmark & \xmark & \xmark & \xmark & \cmark  \\
        \citet{bauerle2021mean} & \cmark & \cmark & \cmark & \cmark & \xmark & \xmark   \\
        \citet{mondal2022approximation, mondal2023mean} & \cmark & \xmark & \cmark & \xmark & \xmark & \cmark   \\
        \citet{motte2022mean, motte2022quantitative} & \cmark & \cmark & \cmark & \cmark & \xmark & \xmark  \\
        our work & \cmark & \cmark & \cmark & \cmark & \cmark & \cmark$^+$  \\
        \bottomrule
    \end{tabular}
\end{table*}

\paragraph{Limitations of standard MFC.} 
However, the strict assumption of only \textit{minor} agents -- i.e. independent, homogeneous agents that can be summarized by their distribution (MF) -- limits applicability. In practice, systems often consist of more than homogeneous agents, and hence one must extend standard MFC towards \textit{major} agents or environment states that are not aggregated. 
For instance, in modelling car traffic on road networks \citep{cabannes2021solving, wu2023leveraging}, when considering only the distribution of cars (\textit{minor} agents) on the network, one cannot model \textit{major} agents or environment states, such as traffic lights or the road conditions respectively. Another example is given by the logistics scenario in Figure~\ref{fig:example} and in the experiments, where many drones on a moving truck collect many packages.

\begin{figure}[t!]
    \centering
    \includegraphics[width=0.99\linewidth]{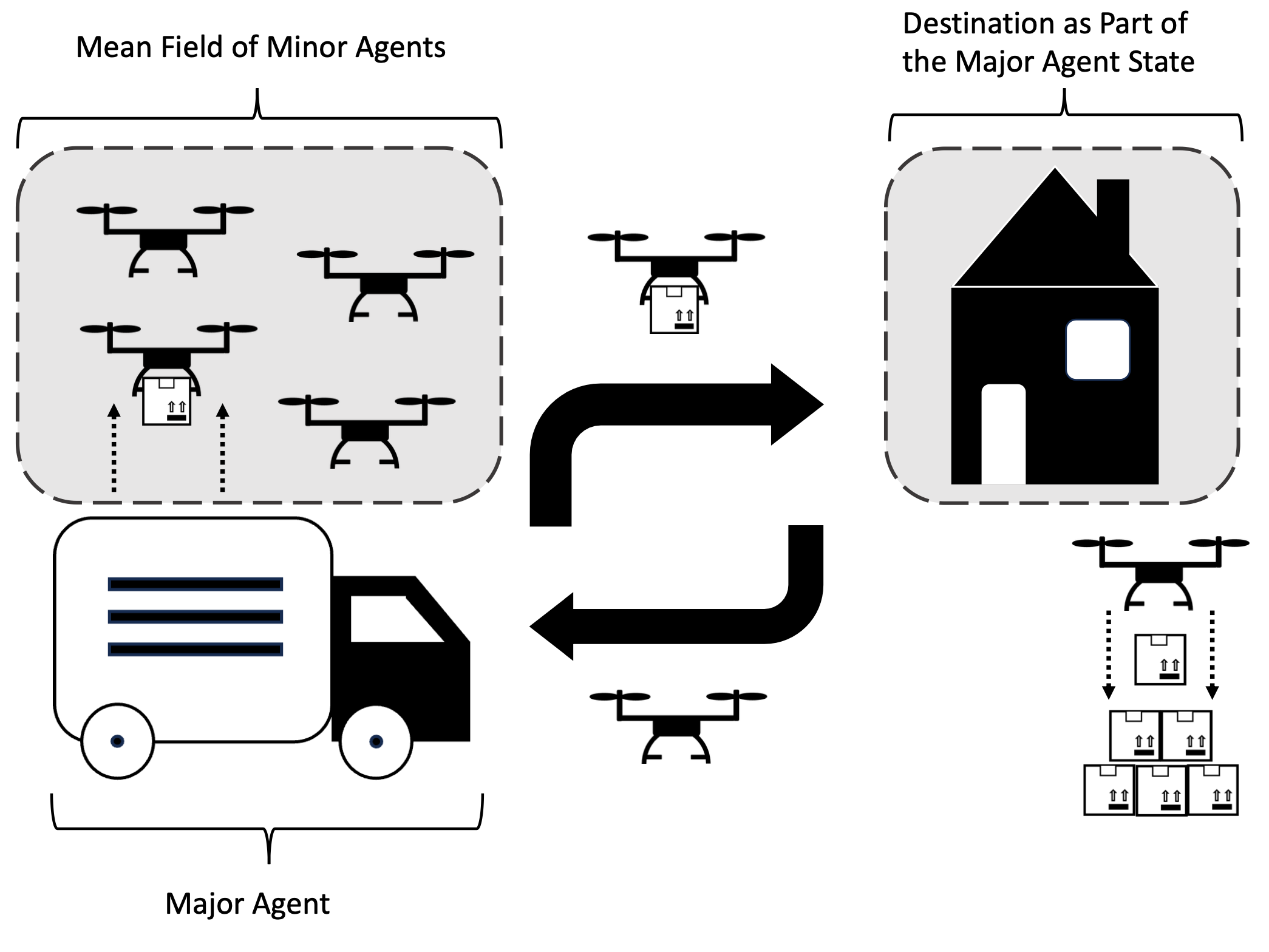}
    \caption{Logistics example: Many drones are modelled as minor agent MF, while truck and package destinations are modelled by a major agent. (See Foraging problem in Section~\ref{sec:problems})}
    \label{fig:example}
\end{figure}

For this purpose, a first step in the continuous-time MFG literature is to consider common noise \citep{carmona2016mean, perrin2020fictitious}, in order to relax the unconditional independence of minor agents. Some more recent works consider such common noise also in discrete-time MFC \citep{carmona2019model, bauerle2021mean, motte2022mean, motte2022quantitative}, or equivalently, global environment states \citep{mondal2023mean}. Essentially, this extension allows MFC to also model random environment effects such as the arrival of new packages in the logistics example (Figure~\ref{fig:example}). \citet{carmona2019model} provide a reformulation of MARL into single-agent RL and consider algorithms for the resulting Markov decision process (MDP). \citet{bauerle2021mean} give approximation theorems and approximate optimality in the finite system by the limiting MFC solution with common noise, and \citet{motte2022mean, motte2022quantitative} quantify the rates of convergence explicitly. See also Table~\ref{tab:diff} for a brief comparison between existing works. In comparison, for the common noise setting, we contribute a new approximation analysis of MFC-based MARL algorithms, where in contrast to prior work, we learn directly with \textit{finite agents}. 

More importantly however, a second contribution is to consider major \textit{agents}. Major agents generalize common noise or environmental states, and take actions that have a non-negligible effect on the system. So far, major agents have only been considered in \textit{continuous}-time, \textit{non}-cooperative MFGs \citep{nourian2013mm, csen2014mean, caines2016mm, sen2016mean}. To the best of our knowledge, no such \textit{discrete}-time, \textit{cooperative} framework has been formulated yet. In this work, we investigate such a framework and associated MARL algorithms.

\begin{figure}[b!]
    \centering
    \includegraphics[width=0.99\linewidth]{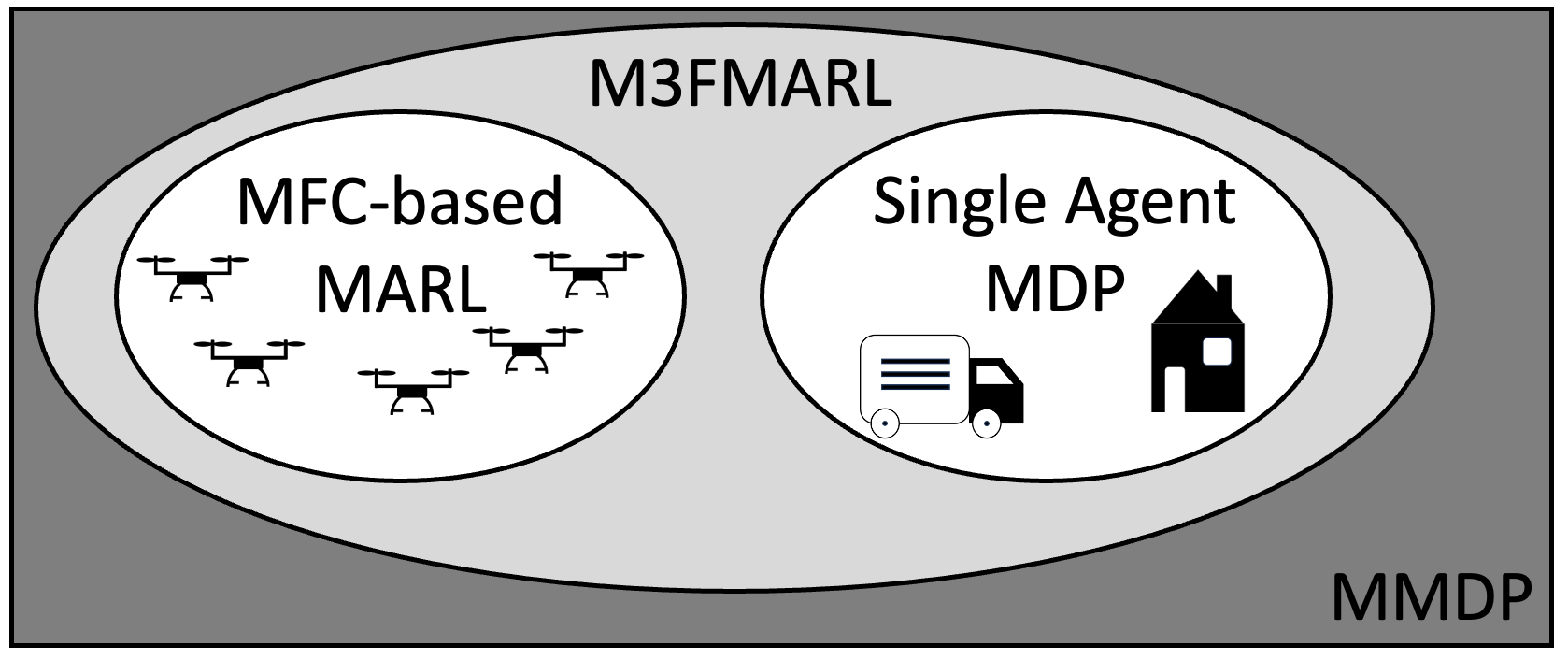}
    \caption{Our M3FC-based MARL generalizes MFC-based MARL and standard single-agent RL in the solution space of general MARL solutions, reducing the otherwise combinatorial nature of MARL \citep{zhang2021multi} to a tractable but still general setting.}
    \label{fig:overview}
\end{figure}

\paragraph{Contribution.}
Existing MFC cannot model general agents and many aggregated agents simultaneously. In essence, we generalize the solution spaces of single-agent RL and MFC-based MARL -- frameworks for cooperative MARL as depicted in Figure~\ref{fig:overview}. This provides both tractability for many aggregated agents and generality for arbitrary general agents. Our contribution is briefly summarized into (i) formulating the first discrete-time MFC model with major agents, together with establishing its theoretical properties; (ii) providing a MFC-based MARL algorithm, which in contrast to prior work learns on the finite problem of interest; and (iii) we perform a significant empirical evaluation, also obtaining positive comparisons of MFC-based MARL against state of the art, whereas prior works on MFC were limited to verifying algorithms on one or two examples.

\section{Major-Minor Mean Field Control} \label{sec:stochastic}
To begin, in this section we extend standard MFC by modelling the presence of a major agent. The generalization to more than one major agent is straightforward. This leads to our discrete-time major-minor MFC (M3FC) model. Overall, we obtain a formulation that allows standard MARL handling of major agents, while tractably handling many minor agents via MFC-based techniques. 

\textit{Notation: By $\E_{X}$ we denote conditional expectations given $X$. The space of probability measures $\mathcal P(\mathcal X)$ on compact metric spaces $\mathcal X$ is equipped with the $1$-Wasserstein distance, unless noted otherwise \citep{villani2009optimal}. Note compactness of $\mathcal P(\mathcal X)$ on compact $\mathcal X$ by Prokhorov's theorem \citep{billingsley2013convergence}. Hence, we sometimes use the uniformly (not Lipschitz) equivalent metric $d_\Sigma(\mu, \mu') \coloneqq \sum_{m=1}^\infty 2^{-m} | \int f_m \, \mathrm d(\mu - \mu') |$, for some sequence of continuous $f_m \colon \mathcal X \to [-1, 1]$ \citep[Theorem~6.6]{parthasarathy2005probability}. } 

\subsection{Finite-Agent System}
Consider $N$ (minor) agents $i \in [N] \coloneqq \{1, \ldots, N\}$ with compact metric state and action spaces $\mathcal X$, $\mathcal U$, equipped with random states and actions $x^{i,N}_t$ and $u^{i,N}_t$ at times $t \in \mathbb N$, where initial states $x^{i,N}_0 \sim \mu_0$ are independently sampled from some initial distribution $\mu_0 \in \mathcal P(\mathcal X)$. In addition to standard MFC, we also consider a single major agent, though the framework can be extended to multiple. Consider major agent state and action spaces, $\mathcal X^0$, $\mathcal U^0$ and state-actions $x^{0,N}_t$, $u^{0,N}_t$, with the major agent formally indexed by $i=0$. Given all actions, the agent states evolve according to kernels $p$, $p^0$ depending on (i) the agent's own state-actions, (ii) the major state-actions, and (iii) the empirical \textbf{MF}, i.e. the $\mathcal P(\mathcal X)$-valued empirical state distribution $\mu^N_t \coloneqq \frac 1 N \sum_{i=1}^N \delta_{x^{i,N}_t}$. This means that minor agents affect other agents only at rate $\frac 1 N$. In practice, we identify minor agents as all agents that matter through their MF $\mu^N_t$. Any remaining agents are major, such that the problem-specific stratification into major and minor agents is always possible.

By symmetry, the system state at any time $t$ is therefore entirely given by $(x^{0,N}_t, \mu^N_t)$. Accordingly, in MFC we share policies between all minor agents. We consider time-variant policies $\pi \in \Pi$, $\pi^0 \in \Pi^0$ from some classes of major and minor policies $\Pi$, $\Pi^0$ that depend on an agent's own state and $(x^{0,N}_t, \mu^N_t)$ at all times $t$. Overall, for all $i \in [N]$ and $t \in \mathbb N$, the finite MFC system follows
\begin{subequations} \label{eq:m3mdp}
\begin{align}
    u^{i,N}_t &\sim \pi_t(u^{i,N}_t \mid x^{i,N}_t, x^{0,N}_t, \mu_t^N), \\
    u^{0,N}_{t} &\sim \pi^0_t(u^{0,N}_{t} \mid x^{0,N}_t, \mu_t^N), \\
    x^{i,N}_{t+1} &\sim p(x^{i,N}_{t+1} \mid x^{i,N}_t, u^{i,N}_t, x^{0,N}_t, u^{0,N}_{t}, \mu_t^N), \\
    x^{0,N}_{t+1} &\sim p^0(x^{0,N}_{t+1} \mid x^{0,N}_t, u^{0,N}_{t}, \mu_t^N) \, .
\end{align}
\end{subequations}

The goal is then to maximize the infinite-horizon discounted objective $J^N(\pi, \pi^0) \coloneqq \mathbb E \left[ \sum_{t=0}^\infty \gamma^t r(x^{0,N}_t, u^{0,N}_{t}, \mu^N_t) \right]$ over minor and major policies $(\pi, \pi^0)$, with discount $\gamma \in (0, 1)$ and reward function $r \colon \mathcal P(\mathcal X) \to \mathbb R$. While an optimal behavior could be learned using standard MARL policy gradient methods, for improved tractability we introduce the following M3FC model in the case of many minor agents.

\begin{remark}
The model is as expressive as in existing MFC \citep{mondal2022approximation, gu2019dynamic}, as it also includes (i) joint state-action MFs $\nu_t \in \mathcal P(\mathcal X \times \mathcal U)$, by splitting time steps in two and defining new states in $\mathcal X \cup \mathcal X \times \mathcal U$, (ii) average rewards over all agents, and (iii) random rewards $r_t^i$ by $r(\mu^N_t) \equiv \frac 1 N \sum_{i=1}^N \E [ r_t^i \mid x^{i,N}_t, \mu^N_t ]$. A finite horizon is handled analogously (without optimal stationary policies).
\end{remark}

\subsection{Mean Field Control Limit}
By the introduction of the MF limit, we obtain a large, more tractable subclass of cooperative multi-agent control problems, which may otherwise suffer from the curse of many agents (combinatorial joint state-action space, \citep{zhang2021multi}). We introduce the MF limit by formally taking $N \to \infty$: The finite-agent control problem is replaced by a higher-dimensional single-agent MDP -- the M3FC MDP. By symmetry, we summarize minor agents into their probability law, the MF $\mu_t \equiv \mathcal L(x^{i,N}_t) \in \mathcal P(\mathcal X)$. It replaces its empirical analogue $\mu^N_t$ by a law of large numbers (LLN). Thus, by definition, the MF $\mu_t$ evolves forward as 
\begin{multline}
    \mu_{t+1} = T(x^0_t, u^0_{t}, \mu_t, \mu_t \otimes \pi_t(\mu_t)) \\
    = \iint p(\cdot \mid x, u, x^0_t, u^0_{t}, \mu_t) \pi_t(\mathrm du \mid x, \mu_t) \mu_t(\mathrm dx),
\end{multline}
with $\pi_t(\mu_t) \coloneqq \pi_t(\cdot \mid \cdot, \mu_t)$, product measures $\mu_t \otimes \pi_t(\mu_t)$ of measure $\mu_t$ and kernel $\pi_t(\mu_t)$ on $\mathcal X \times \mathcal U$, and deterministic dynamics for the MF, $T(x^0, u^0, \mu, h) \coloneqq \iint p(\cdot \mid x, u, x^0, u^0, \mu) h(\mathrm dx, \mathrm du)$. 

Therefore, the state of the limiting system consists only of the MF $\mu_t$ and major state $x^{0}_t$. As a result, we obtain the limiting \textit{M3FC MDP} 
\begin{subequations} \label{eq:m3fc}
\begin{align}
    h_t &\sim \hat \pi_t(h_t \mid x^0_t, \mu_t), \\
    u^0_{t} &\sim \pi^0_t(u^0_{t} \mid x^0_t, \mu_t), \\
    \mu_{t+1} &= T(x^0_t, u^0_{t}, \mu_t, h_t), \\
    x^0_{t+1} &\sim p^0(x^0_{t+1} \mid x^0_t, u^0_{t}, \mu_t)
\end{align}
\end{subequations}
with objective $J(\hat \pi, \pi^0) = \E \left[ \sum_{t=0}^{\infty} \gamma^t r(x^0_t, u^0_t, \mu_t) \right]$ and transition dynamics for the MF $T(x^0, u^0, \mu, h) \coloneqq \iint p(\cdot \mid x, u, x^0, u^0, \mu) h(\mathrm dx, \mathrm du)$. Here, we identify $\mu_t \otimes \pi_t(\mu_t) \equiv h_t \in \mathcal H(\mu_t)$ in the compact set $\mathcal H(\mu) \subseteq \mathcal P(\mathcal X \times \mathcal U)$ of \textit{desired joint state-action distributions with first marginal $\mu$} as part of the action of the M3FC MDP. 

In other words, the action of the M3FC MDP is $(h_t, u^0_t)$ where $h_t$ replaces all the minor agent actions by a LLN. Accordingly, minor agent policies are replaced by MFC policies $\hat \pi$ mapping from current $\mu_t$ to desired state-action distribution $h_t$. The limiting M3FC model abstracts away all the minor agents in the finite system, and considers only the MF and the major agents, as visualized in Figure~\ref{fig:pgm}. The reason for writing joint $h_t$ is mostly technical, as for deterministic $\hat \pi$, we write $\pi_t = \Phi(\hat \pi_t)$ to reobtain agent policies $\mu_t$-a.e. uniquely by disintegration \citep{kallenberg2017random} of $h_t = \hat \pi_t(\mu_t)$ into $\mu_t \otimes \pi'_t$ with decision rule $\pi'_t \in \mathcal P(\mathcal U)^{\mathcal X}$ and using $\pi_t(\mu_t) \equiv \pi'_t$. Inversely, any $\pi \in \Pi$ is represented in the MFC MDP by deterministic $\hat \pi_t = \Phi^{-1}(\pi)_t = \mu_t \otimes \pi_t$.

\begin{remark}
Strictly speaking, in finite-agent control one jointly select actions $(u^{0,N}_t, u^{1,N}_t, \ldots, u^{N,N}_t)$ given joint states $(x^{0,N}_t, x^{1,N}_t, \ldots, x^{N,N}_t)$. But intuitively, (i) joint states reduce to $(x^{0,N}_t, \mu^N_t)$, while (ii) joint actions are replaced by the LLN and sampling actions. Optimality of MFC solutions over larger classes of heterogeneous or joint policies is plausible, but to the best of our knowledge, general result are still limited. See also Appendix~\ref{app:moreopt}.
\end{remark}

For the unfamiliar reader, in Appendix~\ref{app:mfc} we recap basic deterministic MFC without major agents or common noise. There, we recap Lipschitz approximation theorems and dynamic programming principles in compact spaces.

\begin{figure}
    \centering
    \includegraphics[width=0.95\linewidth]{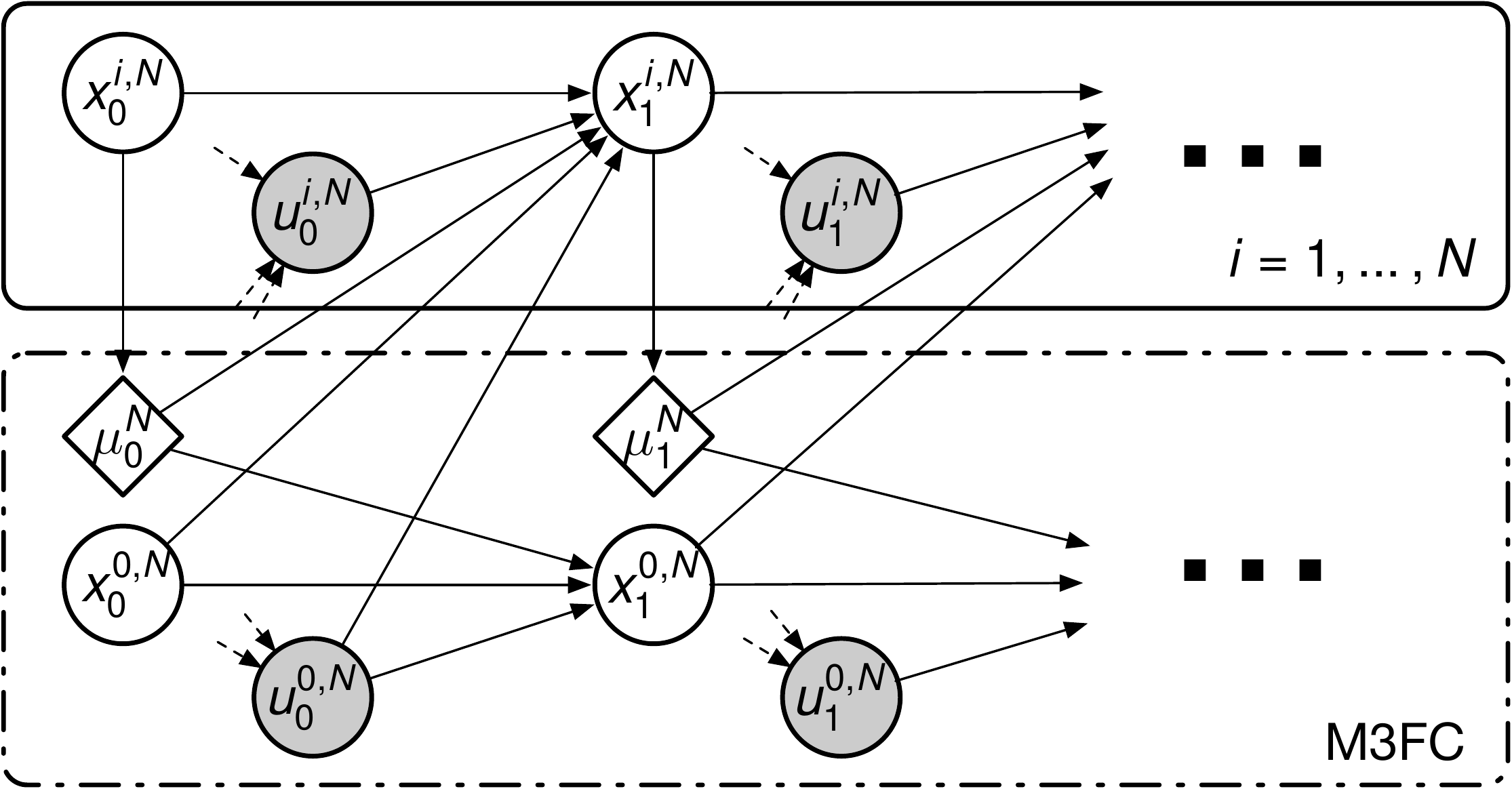}
    \caption{The dynamics \eqref{eq:m3mdp} as a probabilistic graphical model, with actions in grey (inputs omitted for readability). Diamonds denote deterministic functions. M3FC abstracts minor agents $i \in [N]$ by a LLN, considering only their MF as variables in the dotted box.}
    \label{fig:pgm}
\end{figure}

\paragraph{Common noise and global states.} \label{sec:majstate}
In the classical sense \citep{perrin2020fictitious, motte2022mean}, common noise is given by random noise $\epsilon^0_t \sim p_\epsilon(\epsilon^0_t)$ sampled from a fixed distribution $p_\epsilon$, and affects all minor agents at once, $x^{i,N}_{t+1} \sim p(x^{i,N}_{t+1} \mid x^{i,N}_t, u^{i,N}_t, \epsilon^0_t, \mu_t^N)$. This allows to model systems with stochastic MFs and inter-agent correlation, and has added difficulty to the theoretical analysis \citep{carmona2016mean}. Of similar interest are also ``major'' global states $x^{0,N}_t$, which need not be sampled from fixed distributions but evolve dynamically (for MFC with finite global states, see e.g. \citet{mondal2023mean}). 

Both common noise and global states are contained in the M3FC model by using a trivial major agent without actions. We also note that, in general, common noise is equivalent to global states, as global states can be integrated into the minor state conditioned on the common noise. However, for computational purposes the separation of global states and minor agent states can be helpful, as the simplex $\mathcal P(\mathcal X)$ over minor states can be kept smaller for methods based on discretization of the simplex.

\subsection{Dynamic Programming}
As a first step, it is well known that stationary (time-independent) policies suffice for optimality in infinite-horizon discounted MDPs. In the following, this property is also verified for the M3FC MDP. For the following technical results, we assume standard Lipschitz conditions \citep{gu2021mean, mondal2022approximation, pasztor2021efficient}.

\begin{assumption} \label{ass:m3pcont}
The transition kernels $p$, $p^0$ and rewards $r$ are Lipschitz with constants $L_p$, $L_{p^0}$, $L_r$.
\end{assumption}

Assumption~\ref{ass:m3pcont} is true, e.g., in finite spaces if transition matrix entries of $P$ are Lipschitz in the $|\mathcal X|$-dimensional MF vector. The sufficiency of stationary policies is obtained by the dynamic programming principle, which can also be used to compute exact optimal policies in the M3FC MDP. We use the value function $V^*$ as the fixed point of the Bellman equation, $V^*(x^0, \mu) = \max_{(h, u^0) \in \mathcal H(\mu) \times \mathcal U^0} r(x^0, u^0, \mu) + \gamma \mathbb E_{y^0 \sim p^0(y^0 \mid x^0, u^0, \mu)} V^*(y^0, T(x^0, u^0, \mu, h))$.

\begin{theorem} \label{thm:m3dpp}
Under Assumption~\ref{ass:m3pcont}, there exist optimal stationary, deterministic policies $\hat \pi$, $\pi^0$ for the M3FC MDP \eqref{eq:m3fc} by choosing $(\hat \pi(x^0, \mu), \pi^0(x^0, \mu))$ from the maximizers of $\argmax_{(h, u^0) \in \mathcal H(\mu) \times \mathcal U^0} r(x^0, u^0, \mu) + \gamma \mathbb E_{y^0 \sim p^0(y^0 \mid x^0, u^0, \mu)} V^*(y^0, T(x^0, u^0, \mu, h))$.
\end{theorem}

\begin{remark} \label{remark:joint}
We obtain existence of optimal deterministic stationary minor and major policies $\hat \pi$, $\pi^0$ via optimal joint policies $\tilde \pi \equiv \hat \pi \otimes \pi^0$, $(h_t, u^0_t) \sim \tilde \pi((h_t, u^0_t) \mid x^0_t, \mu_t)$.
\end{remark}

The results follow from classical MDP theory \citep{hernandez2012discrete}. Thus, we may solve M3FC problems through the DPP, or approximately by using policy gradients with \textit{stationary} policies for the M3FC MDP, which has naturally continuous actions.

\subsection{Finite Agent Convergence}
Next, in order to show the approximate optimality of M3FC solutions, we first obtain \textbf{propagation of chaos} \citep{sznitman1991topics} -- convergence of empirical MFs to the limiting MF. The result theoretically backs the reduction of multi-agent control to \textit{single-agent} MDPs, as there is no loss of optimality in the finite problem by considering the M3FC problem. We assume standard Lipschitz conditions on policies \citep{gu2021mean, mondal2022approximation, pasztor2021efficient}.

\begin{assumption} \label{ass:m3picont}
The classes of policies $\Pi$, $\Pi^0$ are equi-Lipschitz sets of policies, i.e. there exists $L_\Pi > 0$ such that for all $t$ and $\pi \in \Pi$, $\pi_t \in \mathcal P(\mathcal U)^{\mathcal X \times \mathcal P(\mathcal X)}$ is $L_\Pi$-Lipschitz, and similarly for major policies $\pi^0 \in \Pi^0$.
\end{assumption}

We note that Lipschitz policies are natural, as we usually parametrize policies in a Lipschitz manner; in particular, neural networks allow Lipschitz analysis \citep{pasztor2021efficient, herrera2023local, araujo2023a}. The result is that the limiting system approximates large finite systems.

\begin{theorem} \label{thm:m3muconv}
Fix any family of equi-Lipschitz functions $\mathcal F \subseteq \mathbb R^{\mathcal X^0 \times \mathcal U^0 \times \mathcal P(\mathcal X)}$ with shared Lipschitz constant $L_{\mathcal F}$. Under Assumptions~\ref{ass:m3pcont} and \ref{ass:m3picont}, $(x^{0,N}_t, u^{0,N}_{t}, \mu_t^N)$ converges weakly to $(x^0_t, u^0_{t}, \mu_t)$, uniformly over $f \in \mathcal F$, $(\pi, \pi^0) \in \Pi \times \Pi^0$, $\hat \pi = \Phi^{-1}(\pi)$ at all times $t \in \mathbb N$,
\begin{equation} \label{eq:m3muconv}
    \sup_{f, \pi, \pi^0} \left| \E \left[ f(x^{0,N}_t, u^{0,N}_{t}, \mu_t^N) - f(x^0_t, u^0_{t}, \mu_t) \right] \right| \to 0.
\end{equation}
Further, the convergence rate is $\mathcal O(1 / \sqrt N)$ if $|\mathcal X| < \infty$. 
\end{theorem}

The above motivates M3FC by the following \textbf{near optimality} result of M3FC MDP solutions in the finite system, as it suffices to optimize over stationary M3FC policies. 

\begin{corollary} \label{coro:m3epsopt}
Under Assumptions~\ref{ass:m3pcont} and \ref{ass:m3picont}, optimal deterministic M3FC MDP policies $(\hat \pi^*, \pi^{0*}) \in \argmax_{(\hat \pi, \pi^0)} J(\hat \pi, \pi^0)$ with $\Phi(\hat \pi^*) \in \Pi$ yield $\varepsilon$-optimal $(\Phi(\hat \pi^*), \pi^{0*})$ with $\varepsilon \to 0$ as $N \to \infty$ in the finite system,
$J^N(\Phi(\hat \pi^*), \pi^{0*}) \geq \sup_{(\pi, \pi^0) \in \Pi \times \Pi^0} J^N(\pi, \pi^0) - \varepsilon$.
\end{corollary} 

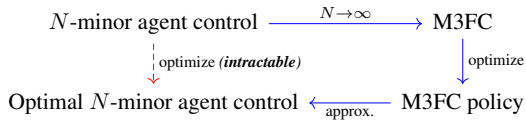
\begin{figure}[b]
    \centering
    \resizebox{0.9\linewidth}{!}{
        $$
        \begin{tikzcd}[column sep=large,ampersand replacement=\&]
        \text{$N$-minor agent control} 
        \arrow[dashed, draw=red]{d}[anchor=west]{\text{optimize \textit{(\textbf{intractable})}}} 
        \arrow[draw=blue]{r}{N \to \infty} 
        \& \text{M3FC} 
        \arrow[draw=blue]{d}[anchor=west]{\text{optimize}} \\
        \text{Optimal $N$-minor agent control} 
        \& \arrow[draw=blue]{l}{\text{approx.}} 
        \text{M3FC policy}
        \end{tikzcd}
        $$
    }
    \caption{Approximation of intractable $N$-agent control by M3FC (blue path), the solution of which is near-optimal for large $N$.}
    \label{fig:detour}
\end{figure}

Therefore, one may solve difficult finite-agent MARL by detouring over the corresponding M3FC MDP as depicted in Figure~\ref{fig:detour}, reducing to an MDP of a complexity independent of the number of agents $N$, which we solve in Section~\ref{sec:algo}.

\section{Major-Minor Mean Field MARL} \label{sec:algo}
As indicated in the prequel and in Figure~\ref{fig:overview}, MARL via M3FC generalizes both single-agent RL and MARL via MFC in the searched policy solution space. Therefore, in M3FC one only optimizes over a tractable, smaller solution space of a single minor and major policy $\Pi, \Pi^0$. At the same time, the framework is highly general and handles arbitrary major agents with many minor agents simultaneously. The reduction of MARL problems to a fixed-complexity single-agent M3FC MDP is the key. In this section, we develop MARL algorithms based on the M3FC framework.

Recalling the motivation of MFC, it is crucial to find tractable sample-based MARL techniques for both complex problems where other methods fail, and for problems where we have no access to the dynamics or reward model. Relating to the former, RL has been applied before to solve MFC given that we know the MFC model equations \citep{carmona2019model, pasztor2021efficient, mondal2022approximation}. However, regarding the latter, we should instead use the MFC formalism to give rise to novel \textit{MARL} algorithms. 

While literature usually focused analysis on the former, in our work we analyze the proposed algorithm not on limiting M3FC MDPs, but on the more interesting finite M3FC system. In particular, if the M3FC MDP is known, one can instantiate finite systems of any size for training. We consider the following perspective: By Theorem~\ref{thm:m3muconv}, the M3FC MDP is approximated well by the finite system. Therefore, we can solve the limiting M3FC MDP by applying our proposed algorithm directly to finite M3FC systems. 

Since we know by Theorem~\ref{thm:m3dpp} that stationary policy suffice, we solve the M3FC MDP \eqref{eq:m3fc} using stationary policies and single-agent RL techniques but on its finite multi-agent instance \eqref{eq:m3mdp}, the combination of which we aptly refer to as Major-Minor Mean Field MARL (M3FMARL). The result is Algorithm~\ref{alg:ppo}, where we directly apply RL to multi-agent systems \eqref{eq:m3mdp} by observing next states $(x^{0,N}_{t+1}, \mu^N_{t+1})$ and rewards $r^N_t \coloneqq r(x^{0,N}_{t}, u^{0,N}_{t}, \mu^N_t)$. The algorithm can be understood as a kind of hierarchical algorithm, as M3FC MDP actions specify behavior for all minor agents at once.

\begin{algorithm}[b!]
    \caption{M3FMARL}
    \label{alg:ppo}
    \begin{algorithmic}[1]
        \FOR {$n=0, 1, \ldots$}
            \FOR {$t = 0, \ldots, B_{\mathrm{len}}-1$}
                \STATE Sample M3FC action from RL policy, i.e. \\$u_t \equiv (u^{0,N}_t, \pi'_t) \sim \tilde \pi^\theta(\cdot \mid x^{0,N}_t, \mu^N_t)$.
                \FOR {$i = 1, \ldots, N$}
                    \STATE Sample $i$-th minor action $u^{i,N}_t \sim \pi'_t(\cdot \mid x^{i,N}_t)$.
                \ENDFOR
                \STATE Execute $\{u^{0,N}_t, u^{1,N}_t, \ldots\}$ for next reward $r^N_t$, state $(x^{0,N}_{t+1}, \mu^N_{t+1})$ and termination $d_{t+1} \in \{0, 1\}$.
            \ENDFOR
            \STATE Perform an update (on policy $\tilde \pi^\theta$) using transitions $B = ((x^{0,N}_t, \mu^N_t), u_t, r^N_t, d_{t+1}, (x^{0,N}_{t+1}, \mu^N_{t+1}))_{t\geq 0}$.
        \ENDFOR
    \end{algorithmic}
\end{algorithm}

\subsection{M3FC-based Policy Gradients} 
The proposed algorithm can be theoretically motivated. As shown in the following, finite-agent policy gradients (PG) estimate the true limiting M3FC MDP PG. First, note that finite state-actions $\mathcal X, \mathcal U$ lead to continuous M3FC MDP actions $\mathcal H(\mu)$, while continuous $\mathcal X, \mathcal U$ even yield infinite-dimensional $\mathcal H(\mu)$. Therefore, we have at least continuous MDPs, complicating value-based learning. 

For this reason, we mainly consider PG methods to solve M3FC-type MARL problems. We parametrize M3FC MDP solutions via RL policies $\tilde \pi^\theta$ with parameters $\theta$, outputting $\xi \in \Xi$ from some compact parameter space $\Xi$ with a Lipschitz map $\Gamma(\xi) = \pi'_t$ to $L_\Pi$-Lipschitz minor agent decision rules $\pi'_t$ (formally, $h_t = \mu_t \otimes \pi'_t$). Assuming the Lipschitzness of the policy network and its gradient in all arguments, on which there has been a great number of recent literature (see e.g. \citet{herrera2023local, araujo2023a} and references therein), we formulate Assumption~\ref{ass:pg}.

\begin{assumption} \label{ass:pg}
    The parameter map $\Gamma$, joint policy $\tilde \pi^\theta$ and log-gradient $\nabla_\theta \log \tilde \pi^\theta$ (or gradient $\nabla_\theta \tilde \pi^\theta$) are $L_\Gamma$, $L_{\tilde \pi}$, $L_{\nabla \tilde \pi}$-Lipschitz and uniformly bounded.
\end{assumption} 

Then, we can apply the PG theorem \citep{sutton1999policy} for the M3FC MDP. The M3FC MDP \eqref{eq:m3fc} essentially substitutes many-agent systems \eqref{eq:m3mdp}, which are natural approximations of the M3FC MDP by Theorem~\ref{thm:m3muconv}. Therefore, we show that M3FMARL (Algorithm~\ref{alg:ppo}) -- single-agent PG on the multi-agent M3FC system -- approximates the true PG of the limiting M3FC MDP, in the case of many minor agents. In other words, M3FMARL solves MARL by approximately solving the single-agent M3FC MDP using policy gradients.

\begin{theorem} \label{thm:pg}
    Under Assumptions~\ref{ass:m3pcont}, \ref{ass:m3picont} and \ref{ass:pg}, the approximate PG of joint policy $\tilde \pi^\theta$ computed on the finite M3FC system \eqref{eq:m3mdp} in Algorithm~\ref{alg:ppo} uniformly tends to the true PG of the M3FC MDP \eqref{eq:m3fc}, as $N \to \infty$.
\end{theorem}

Importantly, the underlying MDP complexity is \textit{independent} of the number of minor agents. Therefore, we would expect Algorithm~\ref{alg:ppo} to be able to perform well in M3FC-type problems, possibly compared to straightforward MARL where each agent is handled separately. Intuitively, for many agents, the reward signal for any single agent can become uninformative: A cooperative, ``averaged'' reward remains almost unaffected by a single agent's actions. This well-known credit assignment issue is therefore solved by the hierarchical structure of M3FC, as credit is assigned to M3FC actions, which affect all minor agents at once and hence receive aggregated credit. Another advantage is that MFC profits from any advances in single-agent RL.

\begin{figure*}[b]
    \centering
    \includegraphics[width=0.99\linewidth]{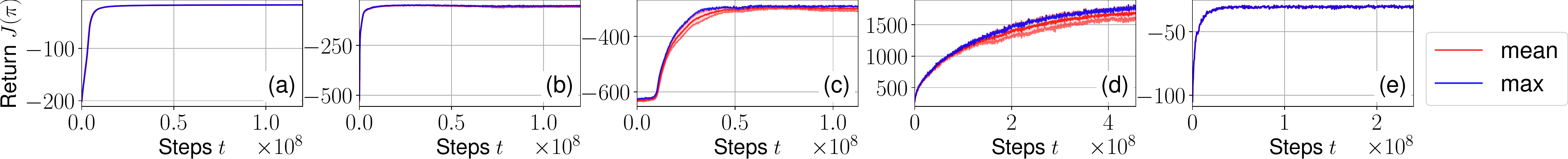}
    \caption{Training curves (mean episode return) of M3FPPO (red), with shaded standard deviation, and maximum (blue) over all three trials (two for Foraging). (a) 2G; (b) Formation; (c) Beach; (d) Foraging; (e) Potential.}
    \label{fig:curvem3fc}
\end{figure*}

\subsection{Implementation Details}
We use the proximal policy optimization (PPO) algorithm \citep{schulman2017proximal} to obtain a M3FC policy $\pi_{\mathrm{RL}}$, instantiating the major minor mean field PPO (M3FPPO) algorithm as an instance of M3FMARL, Algorithm~\ref{alg:ppo}. Other PG algorithms (A2C, leading to M3FA2C) are also compared in our experiments. We parametrize MFs in $\mathcal P(\mathcal X)$ and joint distributions in $\mathcal H(\mu^N_t)$. In practice, for finite $\mathcal X$, $\mathcal U$, the parametrization of $\mathcal P(\mathcal X)$ is immediate by finite-dimensional vectors $\mu^N_t \in \mathcal P(\mathcal X)$. For M3FC actions, consider -- in addition to the major agent action -- the matrix $\xi \in [-1, 1]^{\mathcal X \times \mathcal U}$, which is mapped to probabilities of minor actions in any minor state $\pi'_t(u \mid x) \coloneqq Z^{-1} (\xi_{xu}+1+\epsilon)$, for small $\epsilon = 10^{-10}$ and normalizer $Z$. For continuous $\mathcal X$, $\mathcal U$, we instead partition $\mathcal X$ into $M$ bins and represent $\mu^N_t$ as a histogram, mapping $\xi \in [-1, 1]^{M \times 2}$ to diagonal Gaussian means and standard deviations, $\mu_{\mathcal X_i} \in \mathcal U$, $\sigma_{\mathcal X_i} \in [\epsilon, 0.25 + \epsilon]$, for each of $M$ bins $\mathcal X_i \subseteq \mathcal X$. Major actions $u^{0,N}_t$ are categorical or diagonal Gaussian as usual. For large $\mathcal X, \mathcal U$, one could also consider kernel-based parametrizations \citep{cui2023learning}.

We use two hidden layers of $256$ nodes and $\tanh$ activations for the neural networks of the policies. The neural network policy outputs parameters of a diagonal Gaussian over the major action $u^0$ and matrices $U$ as discussed above. In the discrete Beach scenario below, the neural network instead outputs a categorical distribution using a final softmax layer. We used no GPUs and around 300,000 CPU core hours on Intel Xeon Platinum 9242 CPUs. Optimal transport costs are computed using POT \citep{flamary2021pot}. Our M3FC MDP implementation follows the gym interface \citep{brockman2016openai}, while the implementation of multi-agent RL as in the following fulfills RLlib interfaces \citep{liang2018rllib}. The RL implementations in our work are based on MARLlib 1.0 \citep{hu2022marllib} (MIT license), which uses RLlib 1.8 \citep{liang2018rllib} (Apache-2.0 license) with hyperparameters in Table~\ref{tab:hyperparams}, and otherwise default settings.

\subsection{Comparison to MARL}
The M3FMARL algorithm falls into the paradigm of centralized training with decentralized execution (CTDE) \citep{zhang2021multi}, as we sample a single central M3FC MDP action during training, but enable decentralized execution by sampling $\pi'_t$ separately on each agent instead. For instance, when converged to a deterministic M3FC policy (of which an optimal one is guaranteed to exist by Theorem~\ref{thm:m3dpp}), the M3FC action is always trivially equal for all agents.

Since we also consider continuous minor agent action spaces in our experiments, we compare against PG methods for MARL. In particular, we firstly consider Independent PPO (IPPO), as PPO with independent learning \citep{tan1993multi} and parameter sharing \citep{gupta2017cooperative}, and secondly also Multi-Agent PPO (MAPPO) with centralized critics. The latter has repeatedly shown strong state-of-the-art performance in cooperative MARL \citep{de2020independent, papoudakis2021benchmarking, yu2021surprising}. We also separate major and minor agent policies for improved performance of IPPO / MAPPO. For comparison, we use the same observations for the policy input as in M3FMARL. The policy network architectures match, and the same PPO implementation and hyperparameters are shared with M3FPPO in Table~\ref{tab:hyperparams}. Minor agents are additionally allowed to observe their own states. More details can be found in Appendix~\ref{app:exp}.

\begin{table*}[b]
    \centering
    \caption{Comparison of mean episode returns between best trained policies of standard MARL and M3FMARL methods on a system with $N=20$ agents ($\pm$ $95\%$ confidence interval, for a number of episodes as in Figure~\ref{fig:Ncompare}).}
    \vspace{0.3cm}
    \begin{tabular}{ccccc}
    \toprule
    Problem & IPPO & MAPPO & M3FA2C & M3FPPO  \\
    \hline
    2G  & -43.9 $\pm$ 1.1 & -26.0 $\pm$ 0.5 & -30.6 $\pm$ 0.6 & \textbf{-22.2} $\pm$ 0.56\\
    Formation  &\textbf{-51.1} $\pm$ 2.4 & -101.1 $\pm$ 7.1 & -79.2 $\pm$ 3.1 & -63.9 $\pm$ 4.2 \\
    Beach  &-350.3 $\pm$ 3.4 & -342.9 $\pm$ 4.7 & -424.8 $\pm$ 5.5 & \textbf{-303.5} $\pm$ 3.4\\
    Foraging  &735.3 $\pm$ 46.4 & 803.9 $\pm$ 54.6 & 1398.0 $\pm$ 57.1 & \textbf{1479.4} $\pm$ 36.3 \\
    Potential  &-27.1 $\pm$ 1.4 & \textbf{-26.7} $\pm$ 1.7 & -50.4 $\pm$ 5.5 & -31.3 $\pm$ 1.3 \\
    \bottomrule
    \end{tabular}
    \label{table:compare}
\end{table*}

\begin{figure*}[t]
    \centering
    \includegraphics[width=0.99\linewidth]{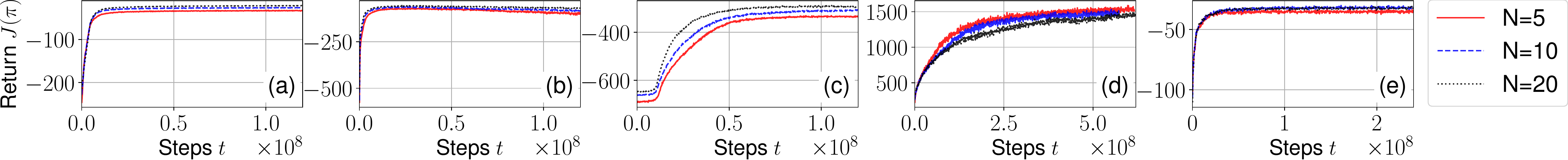}
    \caption{Training curves (mean episode return vs. time steps) of M3FPPO, trained on the finite systems with $N \in \{ 5, 10, 20 \}$. (a) 2G; (b) Formation; (c) Beach; (d) Foraging; (e) Potential.}
    \label{fig:Ntrain}
\end{figure*}

\begin{figure*}[t]
    \centering
    \includegraphics[width=0.99\linewidth]{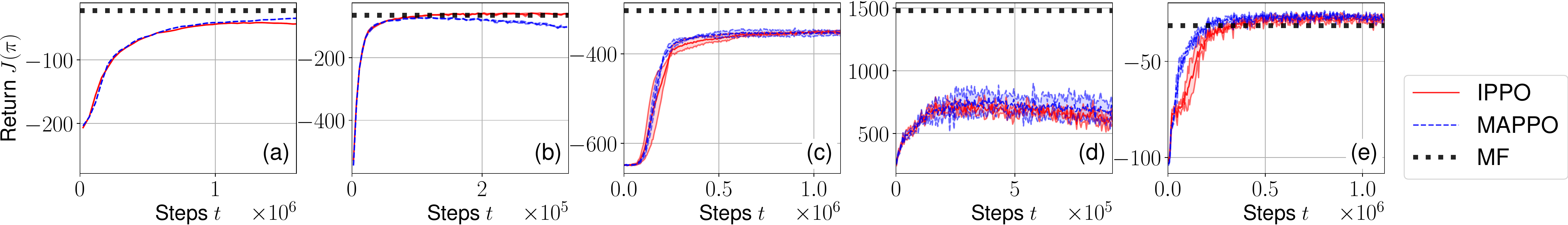}
    \caption{Comparing IPPO / MAPPO vs. results of M3FPPO (MF, ours), as in Figure~\ref{fig:curvem3fc} (no maxima, $N=20$).}
    \label{fig:marl}
\end{figure*}

\setcounter{figure}{8}    
\begin{figure*}[b!]
    \centering
    \includegraphics[width=0.99\linewidth]{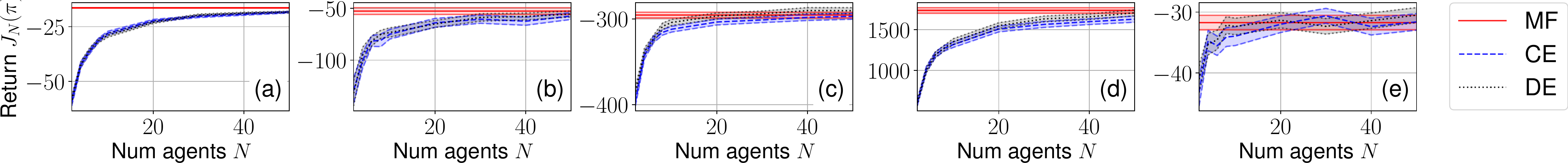}
    \caption{Mean episode return of M3FC policy in finite systems as in Figure~\ref{fig:curvem3fc} over (a-c) $100$, (d) $300$ or (e) $500$ trials (95\% confidence interval shaded). MF: CE, $N=500$; CE / DE: centralized / decentralized execution.}
    \label{fig:Ncompare}
\end{figure*}

\section{Experiments} \label{sec:exp}
In this section, we demonstrate the performance of M3FPPO on illustrative, practical problems. Unless noted otherwise, we use $M=49$ bins ($M=7$ in \textit{Potential}), train for around $24$ hours, and train M3FPPO on the finite-agent system \eqref{eq:m3mdp} with $N=300$ minor agents unless noted otherwise (similar results for less agents in Appendix~\ref{app:exp}). Full descriptions and additional experiments and discussions are in Appendix~\ref{app:exp}.

\setcounter{figure}{7}    
\begin{figure}[b!]
    \centering
    \includegraphics[width=0.99\linewidth]{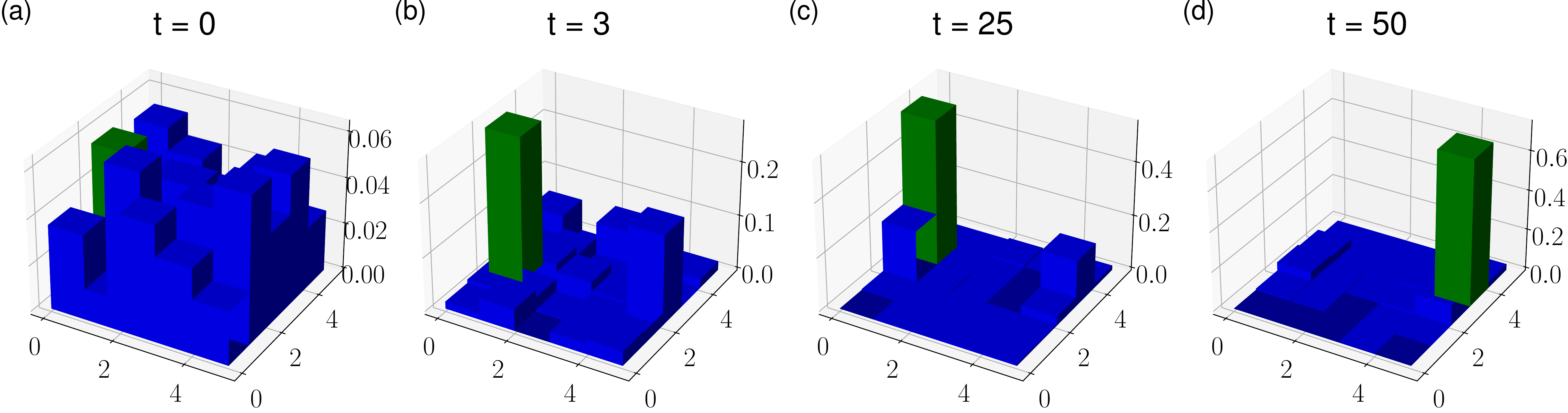}
    \includegraphics[width=0.99\linewidth]{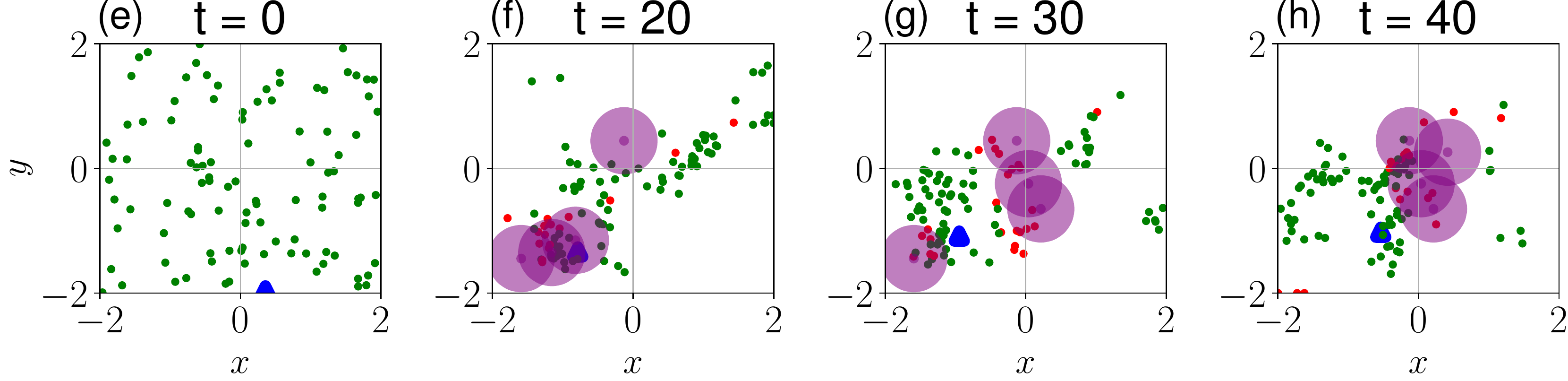}
    \caption{Qualitative visualization of M3FC in Beach (a-d), Foraging (e-h). (a-d): empirical MF, major agent \& target in green; (e-h): blue / green triangle: major agent / target; green / red dots: less- / more-than-half encumbered minor agents; purple: current foraging areas.}
    \label{fig:qualitative2}
\end{figure}
\setcounter{figure}{9}    

\subsection{Problems} \label{sec:problems}
To verify the usefulness of M3FMARL whenever the M3FC model \eqref{eq:m3mdp} is accurate, we consider $5$ benchmark tasks that fulfill the M3FC modelling assumptions. To begin, the simple two Gaussian (\textbf{2G}) problem has no major agent and is equipped with a time-dependent major state: A periodic, time-variant mixture of two Gaussians $\mu^*_t$ -- the major state -- is noisily observed analogously to $\mu^N_t$ via $M=49$ bins. Minor agents should then track the mixture distribution over time, which can find application for example in UAV-based cellular coverage of dynamic users \citep{mozaffari2016efficient}. In the \textbf{Formation} problem, we extend such formation control with major agents. In addition to 2G, one added major agent tracks a moving target. Meanwhile, minor agents instead track a formation around the dynamic major agent, see e.g. \citet{yang2021attacks} for applications. The \textbf{Beach} bar process is a studied classic \citep{arthur1994inductive, perrin2020fictitious}, where minor agents minimize their distances to a bar and additionally avoid crowded areas. Here, the bar moves on a discrete torus. The \textbf{Foraging} problem is archetypal of swarm intelligence \citep{brambilla2013swarm}, and has agents forage randomly generated foraging areas. In particular, we can consider the logistics scenario depicted in Figure~\ref{fig:example}, where a major package truck moves in a restricted space (roads) while minor drones collect packages for urban parcel delivery \citep{marinelli2018route}. Drones fill up at package ``foraging'' areas, and unload near the major agent. Lastly, in the \textbf{Potential} problem, minor agents can generate a potential landscape, the gradient of which pushes the major agent -- e.g., a large object affected by magnetic active matter \citep{jin2021collective} -- to be delivered to a variable target.

\subsection{Evaluation}
In Figure~\ref{fig:curvem3fc}, we see that M3FPPO learning is stable, as M3FPPO reduces hard-to-analyze MARL to single-agent RL, avoiding pathologies of MARL such as non-stationarity of multi-agent learning, or the combinatorial complexity over numbers of agents. In Figure~\ref{fig:Ntrain}, we find similar success in directly training M3FPPO for small $N$ instead of transferring from high $N$. We conclude that M3FPPO remains applicable even with as few as $5$ agents. M3FPPO usually compares well against its A2C variant (M3FA2C) and IPPO / MAPPO, see Table~\ref{table:compare} and Appendix~\ref{app:a2c}. Meanwhile, IPPO / MAPPO under the same hyperparameters as M3FPPO (large batch sizes, see Table~\ref{tab:hyperparams}) can be more unstable and lead to worse results, see Figure~\ref{fig:marl}.

\paragraph{Qualitative behavior.}
In Figure~\ref{fig:qualitative2}, we observe successfully trained behavior in Beach and Foraging: In Beach, M3FPPO learns to accumulate up to $70\%$ of agents on the bar, as more agents on the space lead to a suboptimal reduction in rewards. In Foraging, we find that agents successfully deplete foraging areas shown in the bottom left, moving on afterwards. Further, M3FPPO successfully learns to form mixtures of Gaussians in 2G, a Gaussian around a moving major agent successfully tracking its target in Formation, and similar success in pushing the major agent towards its target in Potential, see Appendix~\ref{app:qual}.

\paragraph{Quantitative support of theory.}
In Figure~\ref{fig:Ncompare}, we \textit{transfer} the trained M3FPPO policy to $N=2, \ldots, 50$, comparing against the performance in the limit ($N=500$). As $N$ grows, the performance converges to the limit, supporting Theorem~\ref{thm:m3muconv} and Corollary~\ref{coro:m3epsopt}. Any sufficiently large system has the same limiting performance as predicted by the theory. We thus have empirical support for scalability, and also transferability between varying numbers of minor agents. 

\paragraph{Comparison to MARL.}
Comparing Figures~\ref{fig:curvem3fc}, \ref{fig:marl} and Table~\ref{table:compare}, we see that (i) by experience sharing, standard MARL can be more sample-efficient, as each step gives $N$ samples instead of just one; and (ii) M3FPPO matches or outperforms IPPO and MAPPO, despite having significantly less control over minor agent actions: All minor agents in a bin (with similar minor agent states) use the same action distributions, which suffices for strong results. 

\paragraph{Decentralized execution.} 
Lastly, decentralized execution by agent-wise randomization -- i.e. sampling M3FC actions per agent instead of a single shared, correlated M3FC action -- has little to no effect, and can even marginally improve performance, see e.g., Beach in Figure~\ref{fig:Ncompare}(c). Figure~\ref{fig:Ncompare} verifies the performance of M3FMARL as a CTDE method. 

\section{Conclusion and Discussion}
We have proposed a generalization of MDPs and MFC, enabling tractable state-of-the-art MARL on general many-agent systems, with both theoretical and empirical support. Beyond the current model and its optimality guarantees, one could work on extended optimality conjectures in Appendix~\ref{app:moreopt}, refined approximations \citep{gast2018refined}, and local interactions \citep{qu2020scalable}. Algorithmically, M3FC MDP actions $\mathcal H(\mu)$ could move beyond binning $\mathcal X$ to gain performance, e.g. via kernels. Lastly, one may try to quantify convergence to the rate $\mathcal O(1/\sqrt N)$ for non-finite $\mathcal X$, as the current proof strategy would need hard-to-verify or unrealistic $d_\Sigma$-Lipschitzness.

\section*{Broader Impact}
This paper presents work whose goal is to advance the field of Machine Learning. There are many potential societal consequences of our work, none which we feel must be specifically highlighted here.

\section*{Acknowledgements}
This work has been co-funded by the LOEWE initiative (Hesse, Germany) within the emergenCITY center, and the Hessian Ministry of Science and the Arts (HMWK) within the projects ``The Third Wave of Artificial Intelligence - 3AI'' and hessian.AI. The authors acknowledge the Lichtenberg high performance computing cluster of the TU Darmstadt for providing computational facilities for the calculations of this research. We thank anonymous reviewers for their helpful comments to improve the manuscript.

\bibliography{main}
\bibliographystyle{plainnat}

\newpage
\appendix
\onecolumn

\section{Related Work}
In this section, we provide additional context on related works. Since the introduction of MFGs in continuous and discrete time \citep{huang2006large, lasry2007mean, saldi2018markov}, MFGs have been studied in various forms, ranging from partially observed systems \citep{saldi2019partially, sen2019mean} over learning-based solutions \citep{guo2019learning, perrin2020fictitious, cui2021approximately, guo2020general, perolat2021scaling, perrin2021generalization, yardim2022policy} on graphs \citep{caines2019graphon, tchuendom2021critical, cui2022learning, hu2022graphon} to considering correlated equilibria \citep{muller2021learning, campi2022correlated, bonesini2022correlated}. 

While many works focus on non-cooperative settings with self-interested agents, this can run counter to the goal of engineering many-agent behavior, e.g., achieving cooperative behavior in swarms of drones. Instead, we focus on the related setting of cooperative MFC \citep{pham2018bellman, gu2019dynamic, mondal2022approximation}, see also work on differential \citep{carmona2018probabilistic}, static \citep{sanjari2020optimal}, or discrete-time deterministic MFC \citep{gast2011mean}. For the unfamiliar reader, we point towards many extensive surveys on the topic of mean field systems \citep{bensoussan2013mean, carmona2018probabilistic, lauriere2022learning}.

In general comparison, another well-known line of mean field MARL \citep{yang2018mean, ganapathi2020multi, subramanian2020partially, subramanian2022decentralized} focuses on approximating the influence of other agents on any particular agent by their average actions. Relatedly, some MARL algorithms introduce approximations over agent neighborhoods based on exponential decay \citep{qu2020scalable, qu2020scalable2, liu2022scalable}. In contrast, MFC assumes dependence on the entire distribution of agents and not, e.g., pairwise terms for each neighbor, per agent.

\section{Deterministic Mean Field Control} \label{app:mfc}
In the following, we provide proofs that were omitted in the main text. To begin, in this section we recap standard deterministic MFC. Here, our general proof technique is introduced. It generalizes to the M3FC case and allows approximation properties and dynamic programming principles beyond finite spaces and Lipschitz continuity assumptions in compact spaces, for MFC models under simple continuity. In standard MFC, we have the model without major agents,
\begin{align} \label{eq:mmdp}
    u^{i,N}_t &\sim \pi_t(u^{i,N}_t \mid x^{i,N}_t, \mu_t^N), \\
    x^{i,N}_{t+1} &\sim p(x^{i,N}_{t+1} \mid x^{i,N}_t, u^{i,N}_t, \mu_t^N)
\end{align}
while in the limit, we have the MF evolution
\begin{align}
    \mu_{t+1} = T(\mu_t, \mu_t \otimes \pi_t(\mu_t)) \coloneqq \iint p(\cdot \mid x, u, \mu_t) \pi_t(\mathrm du \mid x, \mu_t) \mu_t(\mathrm dx)
\end{align}
and MFC system
\begin{align} \label{eq:mfc}
    h_t &\sim \hat \pi_t(h_t \mid \mu_t), \quad
    \mu_{t+1} = T(\mu_t, h_t)
\end{align}
with objective $J(\hat \pi) = \E \left[ \sum_{t=0}^{\infty} \gamma^t r(\mu_t) \right]$.

\paragraph{Dynamic Programming and Propagation of Chaos}
We may solve the hard finite-agent system \eqref{eq:mmdp} near-optimally by instead solving the MFC MDP, allowing direct application of single-agent RL to the MFC MDP with approximate optimality in large systems. Mild continuity assumptions are required.

\begin{assumption} \label{ass:pcont}
The transition kernel $p$ and reward $r$ are continuous.
\end{assumption}
\begin{assumption} \label{ass:picont}
The considered class of policies $\Pi$ is equi-Lipschitz, i.e. there exists $L_\Pi > 0$ such that for all $t$ and $\pi \in \Pi$, $\pi_t \in \mathcal P(\mathcal U)^{\mathcal X \times \mathcal P(\mathcal X)}$ is $L_\Pi$-Lipschitz.
\end{assumption}

We note that Assumption~\ref{ass:pcont} holds true in studied finite spaces, if each transition matrix entry of $P$ is continuous in the $|\mathcal X|$-dimensional MF vector on the simplex (but not necessarily Lipschitz as in \citep{gu2021mean, mondal2022approximation}, the conditions of which we relax for deterministic MFC).

We show a dynamic programming principle \citep{hernandez2012discrete} to solve for and show existence of a deterministic, stationary optimal policy via the value function $V^*$ as the fixed point of the Bellman equation $V^*(\mu) = \max_{h \in \mathcal H(\mu)} r(\mu) + \gamma V^*(T(\mu, h))$.

\begin{theorem} \label{thm:dpp}
Under Assumptions~\ref{ass:pcont}, there exists an optimal stationary, deterministic policy $\hat \pi$ for \eqref{eq:mfc}, with $\hat \pi(\mu) \in \argmax_{h \in \mathcal H(\mu)} r(\mu) + \gamma V^*(T(\mu, h))$.
\end{theorem}

This DPP can be used for computing solutions or to show optimality of stationary policies and existence of an optimum. Next, we show \textbf{propagation of chaos} \citep{sznitman1991topics}. Here, prior proof techniques \citep{gu2021mean, mondal2022approximation} are extended by our approach from \textit{finite} to general \textit{compact} spaces.

\begin{theorem} \label{thm:muconv}
Fix any family of equicontinuous functions $\mathcal F \subseteq \mathbb R^{\mathcal P(\mathcal X)}$. Under Assumptions~\ref{ass:pcont} and \ref{ass:picont}, the empirical MF converges weakly, uniformly over $f \in \mathcal F$, $\pi \in \Pi$, $\hat \pi = \Phi^{-1}(\pi)$, to the limiting MF at all times $t \in \mathbb N$, $\sup_{\pi \in \Pi} \sup_{f \in \mathcal F} \left| \E \left[ f(\mu^N_{t}) \right] - \E \left[ f(\mu_{t}) \right] \right| \to 0$.
\end{theorem}

Importantly, propagation of chaos allows one to show approximate optimality of MFC policies in the large finite control problem, which is of practical relevance for solving many-agent problems.

\begin{corollary} \label{coro:epsopt}
Under Assumptions~\ref{ass:pcont} and \ref{ass:picont}, an optimal deterministic MFC policy  $\pi^* \in \argmax_{\hat \pi} J(\hat \pi)$ yields $\varepsilon$-optimal finite-agent policy $\Phi(\pi^*) \in \Pi$, $J^N(\Phi(\pi^*)) \geq \sup_{\pi \in \Pi} J^N(\pi) - \varepsilon$, with $\varepsilon \to 0$ as $N \to \infty$.
\end{corollary}

\section{Continuity of MF dynamics}
First, we find continuity of the MFC dynamics $T$, which is used in the following proofs.

\begin{lemma} \label{lem:Tcont}
Under Assumption~\ref{ass:pcont}, we have $T(\mu_n, \nu_n) \to T(\mu, \nu)$ whenever $(\mu_n, \nu_n) \to (\mu, \nu)$, 
\end{lemma}
\begin{proof}
To show $T(\mu_n, \nu_n) \to T(\mu, \nu)$, consider any Lipschitz and bounded $f$ with Lipschitz constant $L_f$, then
\begin{align*}
    &\left| \int f \, \mathrm d(T(\mu_n, \nu_n) - T(\mu, \nu)) \right| \\
    &= \left| \iiint f(x') p(\mathrm dx' \mid x, u, \mu_n) \nu_n(\mathrm dx, \mathrm du) - \iiint f(x') p(\mathrm dx' \mid x, u, \mu) \nu(\mathrm dx, \mathrm du) \right| \\
    &\quad \leq \iint \left| \int f(x') p(\mathrm dx' \mid x, u, \mu_n) - \int f(x') p(\mathrm dx' \mid x, u, \mu) \right| \nu_n(\mathrm dx, \mathrm du) \\
    &\qquad + \left| \iiint f(x') p(\mathrm dx' \mid x, u, \mu) (\nu_n(\mathrm dx, \mathrm du) - \nu(\mathrm dx, \mathrm du)) \right| \\
    &\quad \leq \sup_{x \in \mathcal X, u \in \mathcal U} L_f W_1(p(\cdot \mid x, u, \mu_n), p(\cdot \mid x, u, \mu)) \\
    &\qquad + \left| \iiint f(x') p(\mathrm dx' \mid x, u, \mu) (\nu_n(\mathrm dx, \mathrm du) - \nu(\mathrm dx, \mathrm du)) \right| \to 0
\end{align*}
for the first term by $1$-Lipschitzness of $\frac{f}{L_f}$ and Assumption~\ref{ass:pcont} (with compactness implying the uniform continuity), and for the second by $\nu_n \to \nu$ and from continuity by the same argument of $(x,u) \mapsto \iint f(x') p(\mathrm dx' \mid x, u, \mu)$.
\end{proof}

\section{Proof of Theorem~\ref{thm:dpp}} \label{app:thm:dpp}
\begin{proof}
The MFC MDP fulfills \citep{hernandez2012discrete}, Assumption~4.2.1. Here, we use \citep{hernandez2012discrete}, Condition~3.3.4(b1) instead of (b2), see also alternatively \citep{hernandez1992discrete}. 

More specifically, for \citep{hernandez2012discrete}, Assumption~4.2.1(a), the cost function $-r$ is continuous by Assumption~\ref{ass:pcont}, therefore also bounded by compactness of $\mathcal P(\mathcal X)$, and finally also inf-compact on the state-action space of the MFC MDP, since for any $\mu \in \mathcal P(\mathcal X)$ the set $\{ h \in \mathcal H(\mu) \mid -r(\mu) \leq c \}$ is trivially given by $\mathcal H(\mu)$ whenever $-r(\mu) \leq c$, and $\emptyset$ otherwise. Here, we show that $\mathcal H(\mu) \subseteq \mathcal P(\mathcal X \times \mathcal U)$ is a closed subset of the compact space $\mathcal P(\mathcal X \times \mathcal U)$ and therefore also compact. Note first that two measures $\mu, \mu' \in \mathcal P(\mathcal X)$ are equal if and only if for all continuous and bounded $f$ we have $\int f \, \mathrm d\mu = \int f \, \mathrm d\mu'$, see e.g. \citep{billingsley2013convergence}, Theorem 1.3. 

Therefore, as $\mathcal H(\mu)$ is defined by its first marginal $\mu$, $\mathcal H(\mu)$ can be written as an intersection
\begin{align*}
    \mathcal H(\mu) = \bigcap_{f \in C_b(\mathcal X)} \left\{ h \in \mathcal P(\mathcal X \times \mathcal U) \innermid \int f \otimes \mathbf 1 \, \mathrm dh = \int f \, \mathrm d\mu \right\}
\end{align*}
of closed sets: Since $h \mapsto \int f \otimes \mathbf 1 \, \mathrm dh$ is continuous, its preimage of the closed set $\{ \int f \, \mathrm d\mu \}$ is closed. Here, $\otimes$ denotes the tensor product of $f$ with the function $\mathbf 1$ equal one, i.e. $f \otimes \mathbf 1$ is the map $(x, u) \mapsto f(x)$.

Similarly, for \citep{hernandez2012discrete}, Assumption~4.2.1(b), the transition dynamics $T$ are weakly continuous, as for any $(\mu_n, \nu_n) \to (\mu, \nu) \in \mathcal P(\mathcal X) \times \mathcal P(\mathcal X \times \mathcal U)$ we have $T(\mu_n, \nu_n) \to T(\mu, \nu)$ by Lemma~\ref{lem:Tcont} and therefore $\int f \, \mathrm d\delta_{T(\mu_n, \nu_n)} = f(T(\mu_n, \nu_n)) \to f(T(\mu, \nu)) = \int f \, \mathrm d\delta_{T(\mu, \nu)}$ for any continuous and bounded $f \colon \mathcal P(\mathcal X) \to \mathbb R$. 

Furthermore, the MFC MDP fulfills \citep{hernandez2012discrete}, Assumption~4.2.2 by boundedness of $r$ from Assumption~\ref{ass:pcont}. Therefore, the desired statement follows from \citep{hernandez2012discrete}, Theorem~4.2.3.
\end{proof}

\section{Proof of Theorem~\ref{thm:muconv}} \label{app:thm:muconv}
\begin{proof}
Note that we can also show the slightly stronger $L_1$ convergence statement with the absolute value inside of the expectation, $\sup_{\pi \in \Pi} \sup_{f \in \mathcal F} \E \left[ \left| f(\mu^N_{t}) - f(\mu_{t}) \right| \right] \to 0$, but since this statement is only true for deterministic MFC, we avoid it here to later extend our proof directly to M3FC.

The statement $\sup_{\pi \in \Pi} \sup_{f \in \mathcal F} \left| \E \left[ f(\mu^N_{t}) \right] - \E \left[ f(\mu_{t}) \right] \right| \to 0$ is shown inductively over $t \geq 0$. At time $t=0$, it holds by the weak LLN argument, see also the first term below. Assuming the statement at time $t$, then for time $t+1$ we have
\begin{align}
    &\sup_{\pi \in \Pi} \sup_{f \in \mathcal F} \left| \E \left[ f(\mu^N_{t+1}) - f(\mu_{t+1}) \right] \right| \nonumber \\
    &\quad \leq \sup_{\pi \in \Pi} \sup_{f \in \mathcal F} \left| \E \left[ f(\mu^N_{t+1}) - f(T(\mu^N_t, \mu^N_t \otimes \pi_t(\mu^N_t))) \right] \right| \label{eq:first} \\
    &\qquad + \sup_{\pi \in \Pi} \sup_{f \in \mathcal F} \left| \E \left[ f(T(\mu^N_t, \mu^N_t \otimes \pi_t(\mu^N_t))) - f(\mu_{t+1}) \right] \right|. \label{eq:second}
\end{align}

For the first term \eqref{eq:first}, first note that by compactness of $\mathcal P(\mathcal X)$, $\mathcal F$ is uniformly equicontinuous, and hence admits a non-decreasing, concave (as in \citep{devore1993constructive}, Lemma~6.1) modulus of continuity $\omega_{\mathcal F} \colon [0, \infty) \to [0, \infty)$ where $\omega_{\mathcal F}(x) \to 0$ as $x \to 0$ and $|f(\mu) - f(\nu)| \leq \omega_{\mathcal F}(W_1(\mu, \nu))$ for all $f \in \mathcal F$. 

We also have uniform equicontinuity of $\mathcal F$ with respect to the space $(\mathcal P(\mathcal X), d_\Sigma)$ instead of $(\mathcal P(\mathcal X), W_1)$, as the identity map $\mathrm{id} \colon (\mathcal P(\mathcal X), d_\Sigma) \to (\mathcal P(\mathcal X), W_1)$ is uniformly continuous (as both $d_\Sigma$ and $W_1$ metrize the topology of weak convergence, and $\mathcal P(\mathcal X)$ is compact), and therefore there exists a modulus of continuity $\tilde \omega$ for the identity map such that for any $\mu, \nu \in (\mathcal P(\mathcal X), d_\Sigma)$, by the prequel
\begin{align*}
    |f(\mu) - f(\nu)| \leq \omega_{\mathcal F}(W_1(\mathrm{id} \, \mu, \mathrm{id} \, \nu)) \leq \omega_{\mathcal F}(\tilde \omega(d_\Sigma(\mu, \nu)))
\end{align*}
with $\tilde \omega_{\mathcal F} \coloneqq \omega_{\mathcal F} \circ \tilde \omega$, which can be replaced by its least concave majorant (again as in \citep{devore1993constructive}, Lemma~6.1).

Therefore, by Jensen's inequality, for \eqref{eq:first} we obtain
\begin{align*}
    &\left| \E \left[ f(\mu^N_{t+1}) - f(T(\mu^N_t, \mu^N_t \otimes \pi_t(\mu^N_t))) \right] \right| \\
    &\quad \leq \E \left[ \tilde \omega_{\mathcal F}(d_\Sigma(\mu^N_{t+1}, T(\mu^N_t, \mu^N_t \otimes \pi_t(\mu^N_t)))) \right] \\
    &\quad \leq \tilde \omega_{\mathcal F} \left( \E \left[ d_\Sigma(\mu^N_{t+1}, T(\mu^N_t, \mu^N_t \otimes \pi_t(\mu^N_t))) \right] \right)
\end{align*}
irrespective of $\pi$, $f$ via concavity of $\tilde \omega_{\mathcal F}$. Introducing for readability $x^N_t \equiv \{x^{i,N}_t\}_{i\in[N]}$, we then obtain
\begin{align*}
    &\E \left[ d_\Sigma(\mu^N_{t+1}, T(\mu^N_t, \mu^N_t \otimes \pi_t(\mu^N_t))) \right] \\
    &\quad = \sum_{m=1}^\infty 2^{-m} \E \left[ \left| \int f_m \, \mathrm d(\mu^N_{t+1} - T(\mu^N_t, \mu^N_t \otimes \pi_t(\mu^N_t))) \right| \right] \\
    &\quad \leq \sup_{m \geq 1} \E \left[ \E_{x^N_t} \left[ \left| \int f_m \, \mathrm d(\mu^N_{t+1} - T(\mu^N_t, \mu^N_t \otimes \pi_t(\mu^N_t))) \right| \right] \right],
\end{align*}
and by the following weak LLN argument, for the squared term and any $f_m$
\begin{align*}
    &\E_{x^N_t} \left[ \left| \int f_m \, \mathrm d(\mu^N_{t+1} - T(\mu^N_t, \mu^N_t \otimes \pi_t(\mu^N_t))) \right| \right]^2 \\
    &\quad = \E_{x^N_t} \left[ \left| \frac 1 N \sum_{i=1}^N \left( f_m(x^{i,N}_{t+1}) - \E_{x^N_t} \left[ f_m(x^{i,N}_{t+1}) \right] \right) \right| \right]^2 \\
    &\quad \leq \E_{x^N_t} \left[ \left| \frac 1 N \sum_{i=1}^N \left( f_m(x^{i,N}_{t+1}) - \E_{x^N_t} \left[ f_m(x^{i,N}_{t+1}) \right] \right) \right|^2 \right] \\
    &\quad = \frac{1}{N^2} \sum_{i=1}^N \E_{x^N_t} \left[ \left( f_m(x^{i,N}_{t+1}) - \E_{x^N_t} \left[ f_m(x^{i,N}_{t+1}) \right] \right)^2 \right] \leq \frac{4}{N} \to 0
\end{align*}
by bounding $|f_m| \leq 1$, as the cross-terms are zero by conditional independence of $x^{i,N}_{t+1}$ given $x^N_t$. By the prequel, the term \eqref{eq:first} hence converges to zero.

For the second term \eqref{eq:second}, we have
\begin{align*}
    &\sup_{\pi \in \Pi} \sup_{f \in \mathcal F} \left| \E \left[ f(T(\mu^N_t, \mu^N_t \otimes \pi_t(\mu^N_t))) - f(\mu_{t+1}) \right] \right| \\
    &\quad = \sup_{\pi \in \Pi} \sup_{f \in \mathcal F} \left| \E \left[ f(T(\mu^N_t, \mu^N_t \otimes \pi_t(\mu^N_t))) - f(T(\mu_t, \mu_t \otimes \pi_t(\mu_t))) \right] \right| \\
    &\quad \leq \sup_{\pi \in \Pi} \sup_{g \in \mathcal G} \left| \E \left[ g(\mu^N_t) - g(\mu_t) \right] \right| \to 0
\end{align*}
by the induction assumption, where we defined $g = f \circ \tilde T^{\pi_t}$ from the class $\mathcal G$ of equicontinuous functions with modulus of continuity $\omega_{\mathcal G} \coloneqq \omega_{\mathcal F} \circ \omega_{T}$, where $\omega_T$ denotes the uniform modulus of continuity of $\mu_t \mapsto \tilde T^{\pi_t}(\mu_t) \coloneqq T(\mu_t, \mu_t \otimes \pi_t(\mu_t)))$ over all policies $\pi$. Here, this equicontinuity of $\{ \tilde T^{\pi_t} \}_{\pi \in \Pi}$ follows from Lemma~\ref{lem:Tcont} and the equicontinuity of functions $\mu_t \mapsto \mu_t \otimes \pi_t(\mu_t)$ due to uniformly Lipschitz $\Pi$ as we show in the following, completing the proof by induction:

Consider $\mu_n \to \mu \in \mathcal P(\mathcal X)$, then we have
\begin{align*}
    &\sup_{\pi \in \Pi} W_1(\mu_n \otimes \pi_t(\mu_n), \mu \otimes \pi_t(\mu)) \\
    &\quad = \sup_{\pi \in \Pi} \sup_{\lVert f' \rVert_{\mathrm{Lip}} \leq 1} \left| \int f' \, \mathrm d(\mu_n \otimes \pi_t(\mu_n) - \mu \otimes \pi_t(\mu)) \right| \\
    &\quad \leq \sup_{\pi \in \Pi} \sup_{\lVert f' \rVert_{\mathrm{Lip}} \leq 1} \left| \iint f'(x, u) (\pi_t(\mathrm du \mid x, \mu_n) - \pi_t(\mathrm du \mid x, \mu)) \mu_n(\mathrm dx) \right| \\
    &\qquad + \sup_{\pi \in \Pi} \sup_{\lVert f' \rVert_{\mathrm{Lip}} \leq 1} \left| \iint f'(x, u) \pi_t(\mathrm du \mid x, \mu) (\mu_n(\mathrm dx) - \mu(\mathrm dx)) \right|
\end{align*}
where for the first term
\begin{align*}
    &\sup_{\pi \in \Pi} \sup_{\lVert f' \rVert_{\mathrm{Lip}} \leq 1} \left| \iint f'(x, u) (\pi_t(\mathrm du \mid x, \mu_n) - \pi_t(\mathrm du \mid x, \mu)) \mu_n(\mathrm dx) \right| \\
    &\quad \leq \sup_{\pi \in \Pi} \sup_{\lVert f' \rVert_{\mathrm{Lip}} \leq 1} \int \left| \int f'(x, u) (\pi_t(\mathrm du \mid x, \mu_n) - \pi_t(\mathrm du \mid x, \mu)) \right| \mu_n(\mathrm dx) \\
    &\quad \leq \sup_{\pi \in \Pi} \sup_{\lVert f' \rVert_{\mathrm{Lip}} \leq 1} \sup_{x \in \mathcal X} \left| \int f'(x, u) (\pi_t(\mathrm du \mid x, \mu_n) - \pi_t(\mathrm du \mid x, \mu)) \right| \\
    &\quad = \sup_{\pi \in \Pi} \sup_{x \in \mathcal X} W_1(\pi_t(\cdot \mid x, \mu_n), \pi_t(\cdot \mid x, \mu)) \\
    &\quad \leq L_\Pi W_1(\mu_n, \mu) \to 0
\end{align*}
by Assumption~\ref{ass:picont}, and similarly for the second by first noting $1$-Lipschitzness of $x \mapsto \int \frac {f'(x, u)} {L_\Pi + 1}  \pi_t(\mathrm du \mid x, \mu)$, as for $y \neq x$
\begin{align}
    &\left| \int \frac {f'(y, u)} {L_\Pi + 1}  \pi_t(\mathrm du \mid y, \mu) - \int \frac {f'(x, u)} {L_\Pi + 1}  \pi_t(\mathrm du \mid x, \mu) \right| \nonumber\\
    &\quad \leq \left| \int \frac {f'(y, u) - f'(x, u)} {L_\Pi + 1}  \pi_t(\mathrm du \mid y, \mu) \right| + \left| \int \frac {f'(x, u)} {L_\Pi + 1} (\pi_t(\mathrm du \mid y, \mu) - \pi_t(\mathrm du \mid x, \mu)) \right| \nonumber\\
    &\quad \leq \frac{1}{L_\Pi + 1} d(y, x) + \frac{1}{L_\Pi + 1} W_1(\pi_t(\cdot \mid y, \mu), \pi_t(\cdot \mid x, \mu)) \nonumber\\
    &\quad \leq \left( \frac{1}{L_\Pi + 1} + \frac{L_\Pi}{L_\Pi + 1} \right) d(x, y) \label{eq:lipschitz-fpi}
\end{align}
with $\frac{1}{L_\Pi + 1} + \frac{L_\Pi}{L_\Pi + 1} = 1 \leq 1$, and therefore again
\begin{align*}
    &\sup_{\pi \in \Pi} \sup_{\lVert f' \rVert_{\mathrm{Lip}} \leq 1} \left| \iint f'(x, u) \pi_t(\mathrm du \mid x, \mu) (\mu_n(\mathrm dx) - \mu(\mathrm dx)) \right| \\
    &\quad = \sup_{\pi \in \Pi} \sup_{\lVert f' \rVert_{\mathrm{Lip}} \leq 1} (L_\Pi + 1) \left| \iint \frac{f'(x, u)}{L_\Pi + 1}  \pi_t(\mathrm du \mid x, \mu) (\mu_n(\mathrm dx) - \mu(\mathrm dx)) \right| \\
    &\quad \leq (L_\Pi + 1) W_1(\mu_n, \mu) \to 0.
\end{align*}
This completes the proof by induction.
\end{proof}

\section{Proof of Corollary~\ref{coro:epsopt}} \label{app:coro:epsopt}
\begin{proof}
First, we show that from uniform convergence in Theorem~\ref{thm:muconv}, the finite-agent objectives converge uniformly to the MFC limit.

\begin{lemma} \label{lem:Jconv}
Under Assumptions~\ref{ass:pcont} and \ref{ass:picont}, the finite-agent objective converges uniformly to the MFC limit,
\begin{equation}
    \sup_{\pi \in \Pi} \left| J^N(\pi) - J(\Phi^{-1}(\pi)) \right| \to 0.
\end{equation}
\end{lemma}
\begin{subproof}
    For any $\varepsilon > 0$, choose time $T \in \mathbb N$ such that $\sum_{t=T}^{\infty} \gamma^t \E \left| \left[ r(\mu^N_t) - r(\mu_t) \right] \right| \leq \frac{\gamma^T}{1-\gamma} \max_\mu 2 |r(\mu)| < \frac \varepsilon 2$. By Theorem~\ref{thm:muconv}, $\sum_{t=0}^{T-1} \gamma^t \E \left| \left[ r(\mu^N_t) - r(\mu_t) \right] \right| < \frac \varepsilon 2$ for sufficiently large $N$. The result follows.
\end{subproof}

The approximate optimality of MFC solutions in the finite system follows immediately: By Lemma~\ref{lem:Jconv}, we have
\begin{align*}
    &J^N(\Phi(\pi^*)) - \sup_{\pi \in \Pi} J^N(\pi) = \inf_{\pi \in \Pi} (J^N(\pi^*) - J^N(\pi)) \\
    &\quad \geq \inf_{\pi \in \Pi} (J^N(\Phi(\pi^*)) - J(\pi^*))
    + \inf_{\pi \in \Pi} (J(\pi^*) - J(\Phi^{-1}(\pi))) + \inf_{\pi \in \Pi} (J(\Phi^{-1}(\pi)) - J^N(\pi)) \\
    &\quad \geq - \frac \varepsilon 2 + 0 - \frac \varepsilon 2 = - \varepsilon
\end{align*}
for sufficiently large $N$, where the second term is zero by optimality of $\pi^*$ in the MFC problem.
\end{proof}

\section{Stochastic Mean Field Control with Common Noise and Major States} \label{app:msmfc} \label{app:majstate}
For convenience, we also restate the results for MFC with major states, or common noise. We have the finite MFC system with major states
\begin{subequations} \label{eq:msmmdp}
\begin{align}
    u^{i,N}_t &\sim \pi_t(u^{i,N}_t \mid x^{i,N}_t, x^{0,N}_t, \mu_t^N), \\
    x^{i,N}_{t+1} &\sim p(x^{i,N}_{t+1} \mid x^{i,N}_t, u^{i,N}_t, x^{0,N}_t, \mu_t^N), \quad
    x^{0,N}_{t+1} \sim p^0(x^{0,N}_{t+1} \mid x^{0,N}_t, \mu_t^N)
\end{align}
\end{subequations}
and objective $J^N(\pi) = \E \left[ \sum_{t=0}^{\infty} \gamma^t r(x^{0,N}_t, \mu^N_t) \right]$ analogous to \eqref{eq:mmdp}, with the corresponding limiting MFC MDP with major states analogous to \eqref{eq:mfc},
\begin{align} \label{eq:msmfc}
    h_t \sim \hat \pi_t(h_t \mid x^0_t, \mu_t), \quad
    \mu_{t+1} = T(x^0_t, \mu_t, h_t), \quad
    x^0_{t+1} \sim p^0(x^0_{t+1} \mid x^0_t, \mu_t)
\end{align}
with objective $J(\hat \pi) = \E \left[ \sum_{t=0}^{\infty} \gamma^t r(x^0_t, \mu_t) \right]$, where $T(x^0, \mu, h) \coloneqq \iint p(\cdot \mid x, u, x^0, \mu) h(\mathrm dx, \mathrm du)$.

\begin{assumption} \label{ass:mspcont}
The transition kernels $p$, $p^0$ and rewards $r$ are Lipschitz continuous with constants $L_p$, $L_{p^0}$, $L_r$.
\end{assumption}
\begin{assumption} \label{ass:mspicont}
The class of policies $\Pi$ are equi-Lipschitz, i.e. there exists $L_\Pi > 0$ such that for all $t$ and $\pi \in \Pi$, $\pi_t \in \mathcal P(\mathcal U)^{\mathcal X \times \mathcal P(\mathcal X)}$ is $L_\Pi$-Lipschitz.
\end{assumption}

\begin{theorem} \label{thm:msdpp}
Under Assumption~\ref{ass:mspcont}, there exists an optimal stationary, deterministic policy $\hat \pi$ for the MFC MDP \eqref{eq:msmfc} by choosing $\hat \pi(x^0, \mu)$ from the maximizers of $\argmax_{h \in \mathcal H(\mu)} r(x^0, \mu) + \gamma \mathbb E_{y^0 \sim p^0(y^0 \mid x^0, \mu)} V^*(y^0, T(x^0, \mu, h))$, with $V^*$ the unique fixed point of the Bellman equation $V^*(x^0, \mu) = \max_{h \in \mathcal H(\mu)} r(x^0, \mu) + \gamma \mathbb E_{y^0 \sim p^0(y^0 \mid x^0, \mu)} V^*(y^0, T(x^0, \mu, h))$ (value function).
\end{theorem}

\begin{theorem} \label{thm:msmuconv}
Fix any family of equi-Lipschitz functions $\mathcal F \subseteq \mathbb R^{\mathcal X^0 \times \mathcal P(\mathcal X)}$ with shared Lipschitz constant $L_{\mathcal F}$ for all $f \in \mathcal F$. Under Assumption~\ref{ass:mspcont}, the random variable $(x^{0,N}_t, \mu_t^N)$ converges weakly, uniformly over $\mathcal F$, $\Pi$, to $(x^0_t, \mu_t)$ at all times $t \in \mathbb N$,
\begin{equation} \label{eq:msmuconv}
    \sup_{\pi \in \Pi} \sup_{f \in \mathcal F} \left| \E \left[ f(x^{0,N}_t, \mu_t^N) - f(x^0_t, \mu_t) \right] \right| \to 0.
\end{equation}
\end{theorem}

\begin{corollary} \label{coro:msepsopt}
Under Assumptions~\ref{ass:mspcont} and \ref{ass:mspicont}, optimal deterministic MFC policies $\pi^* \in \argmax_{\pi} J(\pi)$ result in $\varepsilon$-optimal policies $\Phi(\pi^*)$ in the finite-agent problem with $\varepsilon \to 0$ as $N \to \infty$,
\begin{equation}
    J^N(\Phi(\pi^*)) \geq \sup_{\pi \in \Pi} J^N(\pi) - \varepsilon.
\end{equation}
\end{corollary}

The proofs and interpretation are directly analogous to the M3FC case and the following proofs, by leaving out the major agent actions, or alternatively using the M3FC results with a trivial singleton major action space, $|\mathcal U^0| = 1$.

\section{Proof of Theorem~\ref{thm:m3dpp}} \label{app:thm:m3dpp}
\begin{proof}
The proof is analogous to Appendix~\ref{app:thm:dpp} by first showing the continuity of $T$ (proof further below).

\begin{lemma} \label{lem:m3Tcont}
Under Assumption~\ref{ass:m3pcont}, for any sequence $(x^0_n, u^0_n, \mu_n, \nu_n) \to (x^0, u^0, \mu, \nu) \in \mathcal X^0 \times \mathcal U^0 \times \mathcal P(\mathcal X) \times \mathcal P(\mathcal X \times \mathcal U)$, we have $T(x^0_n, u^0_n, \mu_n, \nu_n) \to T(x^0, u^0, \mu, \nu)$.
\end{lemma}

For \citep{hernandez2012discrete}, Assumption~4.2.1(a), the cost function $-r$ is continuous by Assumption~\ref{ass:m3pcont}, therefore also bounded by compactness of $\mathcal X^0 \times \mathcal P(\mathcal X)$, and finally also inf-compact on the state-action space of the M3FC MDP, since for any $(x^0, \mu) \in \mathcal X^0 \times \mathcal P(\mathcal X)$ the set $\{ (h, u^0) \in \mathcal H(\mu) \times \mathcal U^0 \mid -r(x^0, u^0, \mu) \leq c \}$ is given by $\mathcal H(\mu) \times \tilde r^{-1}((-\infty, c])$, where we defined $\tilde r(u^0) \coloneqq -r(x^0, u^0, \mu)$. Note that $\mathcal H(\mu)$ is compact by the same argument as in Appendix~\ref{app:thm:dpp}, while $\tilde r$ is continuous by Assumption~\ref{ass:m3pcont} and therefore its preimage of the closed set $(-\infty, c]$ is compact.

For \citep{hernandez2012discrete}, Assumption~4.2.1(b), consider any continuous and bounded $f \colon \mathcal X^0 \times \mathcal P(\mathcal X) \to \mathbb R$. The continuity is uniform by compactness. Hence, $\sup_{x' \in \mathcal X^0} \left| f(x', \mu'_n) - f(x', \mu') \right| \to 0$ as $\mu'_n \to \mu' \in \mathcal P(\mathcal X)$. Thus, whenever $(x^0_n, u^0_n, \mu_n, \nu_n) \to (x^0, u^0, \mu, \nu) \in \mathcal X^0 \times \mathcal U^0 \times \mathcal P(\mathcal X) \times \mathcal P(\mathcal X \times \mathcal U)$, we have
\begin{align*}
    &\left| \iint f(x', \mu) \, \delta_{T^*_n}(\mathrm d\mu') \, p^0(\mathrm dx' \mid x^0_n, u^0_n, \mu_n) - \iint f(x', \mu) \, \delta_{T^*}(\mathrm d\mu') \, p^0(\mathrm dx' \mid x^0, u^0, \mu) \right| \\
    &\quad = \left| \int f(x', T^*_n) \, p^0(\mathrm dx' \mid x^0_n, u^0_n, \mu_n) - \int f(x', T^*) \, p^0(\mathrm dx' \mid x^0, u^0, \mu) \right| \\
    &\quad \leq \left| \int f(x', T^*_n) \, p^0(\mathrm dx' \mid x^0_n, u^0_n, \mu_n) - \int f(x', T^*) \, p^0(\mathrm dx' \mid x^0_n, u^0_n, \mu_n) \right| \\
    &\qquad + \left| \int f(x', T^*) \, p^0(\mathrm dx' \mid x^0_n, u^0_n, \mu_n) - \int f(x', T^*) \, p^0(\mathrm dx' \mid x^0, u^0, \mu) \right| \\
    &\quad \leq \sup_{x' \in \mathcal X^0} \left| f(x', T^*_n) - f(x', T^*) \right| \\
    &\qquad + \left| \int \tilde f(x') \, p^0(\mathrm dx' \mid x^0_n, u^0_n, \mu_n) - \int \tilde f(x') \, p^0(\mathrm dx' \mid x^0, u^0, \mu) \right| \to 0
\end{align*}
for the first term by the prequel where $T^*_n \coloneqq T(x^0_n, u^0_n, \mu_n, \nu_n) \to T^* \coloneqq T(x^0, u^0, \mu, \nu)$ by Lemma~\ref{lem:m3Tcont}, and for the second term by applying Assumption~\ref{ass:m3pcont} to $\tilde f(x') \coloneqq f(x', T^*)$. This shows weak continuity of the dynamics.

Furthermore, the M3FC MDP fulfills \citep{hernandez2012discrete}, Assumption~4.2.2 by boundedness of $r$ from Assumption~\ref{ass:m3pcont}. Therefore, the desired statement follows from \citep{hernandez2012discrete}, Theorem~4.2.3.
\end{proof}

\section{Proof of Lemma~\ref{lem:m3Tcont}}
\begin{proof}
To show $T(x^0_n, u^0_n, \mu_n, \nu_n) \to T(x^0, u^0, \mu, \nu)$, consider any Lipschitz and bounded $f$ with Lipschitz constant $L_f$, then
\begin{align*}
    &\left| \int f \, \mathrm d(T(x^0_n, u^0_n, \mu_n, \nu_n) - T(x^0, u^0, \mu, \nu)) \right| \\
    &= \left| \iiint f(x') \left( p(\mathrm dx' \mid x, u, x^0_n, u^0_n, \mu_n) \nu_n(\mathrm dx, \mathrm du) - p(\mathrm dx' \mid x, u, x^0, u^0, \mu) \nu(\mathrm dx, \mathrm du) \right) \right| \\
    &\quad \leq \iint \left| \int f(x') p(\mathrm dx' \mid x, u, x^0_n, u^0_n, \mu_n) - \int f(x') p(\mathrm dx' \mid x, u, x^0, u^0, \mu) \right| \nu_n(\mathrm dx, \mathrm du) \\
    &\qquad + \left| \iiint f(x') p(\mathrm dx' \mid x, u, x^0, u^0, \mu) (\nu_n(\mathrm dx, \mathrm du) - \nu(\mathrm dx, \mathrm du)) \right| \\
    &\quad \leq \sup_{x \in \mathcal X, u \in \mathcal U} L_f W_1(p(\cdot \mid x, u, x^0_n, u^0_n, \mu_n), p(\cdot \mid x, u, x^0, u^0, \mu)) \\
    &\qquad + \left| \iiint f(x') p(\mathrm dx' \mid x, u, x^0, u^0, \mu) (\nu_n(\mathrm dx, \mathrm du) - \nu(\mathrm dx, \mathrm du)) \right| \to 0
\end{align*}
for the first term by $1$-Lipschitzness of $\frac{f}{L_f}$ and Assumption~\ref{ass:m3pcont} (with compactness implying the uniform continuity), and for the second by $\nu_n \to \nu$ and continuity of $(x,u) \mapsto \iint f(x') p(\mathrm dx' \mid x, u, x^0, u^0, \mu)$ by the same argument.
\end{proof}

\section{Proof of Theorem~\ref{thm:m3muconv}} \label{app:thm:m3muconv}
\begin{proof}
The statement $\sup_{f, \pi, \pi^0} \left| \E \left[ f(x^{0,N}_t, u^{0,N}_{t}, \mu_t^N) - f(x^0_t, u^0_{t}, \mu_t) \right] \right|$ is shown inductively over $t \geq 0$. At time $t=0$, it holds by the weak LLN argument, see also the first term below. Assuming the statement at time $t$, then for time $t+1$ we have
\begin{align}
    &\sup_{(\pi, \pi^0) \in \Pi \times \Pi^0} \sup_{f \in \mathcal F} \left| \E \left[ f(x^{0,N}_{t+1}, u^{0,N}_{t+1}, \mu^N_{t+1}) - f(x^0_{t+1}, u^0_{t+1}, \mu_{t+1}) \right] \right| \nonumber \\
    &\quad \leq \sup_{\pi, \pi^0} \sup_{f \in \mathcal F} \left| \E \left[ f(x^{0,N}_{t+1}, u^{0,N}_{t+1}, \mu^N_{t+1}) - f(x^{0,N}_{t+1}, u^{0,N}_{t+1}, \hat \mu^N_{t+1}) \right] \right| \label{eq:m3first} \\
    &\qquad + \sup_{\pi, \pi^0} \sup_{f \in \mathcal F} \left| \E \left[ f(x^{0,N}_{t+1}, u^{0,N}_{t+1}, \hat \mu^N_{t+1}) - f(x^0_{t+1}, u^0_{t+1}, \mu_{t+1}) \right] \right| \label{eq:m3second}
\end{align}
where for readability, we again write $\pi_t(x^0_t, \mu_t) \coloneqq \pi_t(\cdot \mid \cdot, x^0_t, \mu_t)$ and introduce the random variable
\begin{align*}
    \hat \mu^N_{t+1} \coloneqq T(x^{0,N}_{t}, u^{0,N}_{t}, \mu^N_t, \mu^N_t \otimes \pi_t(x^{0,N}_{t}, \mu^N_t)).
\end{align*}

By compactness of $\mathcal X^0 \times \mathcal U^0 \times \mathcal P(\mathcal X)$, $\mathcal F$ is uniformly equicontinuous, and hence admits a non-decreasing, concave (as in \citep{devore1993constructive}, Lemma~6.1) modulus of continuity $\omega_{\mathcal F} \colon [0, \infty) \to [0, \infty)$ where $\omega_{\mathcal F}(x) \to 0$ as $x \to 0$ and $|f(x, u, \mu) - f(x', u', \nu)| \leq \omega_{\mathcal F}(d(x, x') + d(u, u') + W_1(\mu, \nu))$ for all $f \in \mathcal F$, and analogously there exists such $\tilde \omega_{\mathcal F}$ with respect to $(\mathcal P(\mathcal X), d_\Sigma)$ instead of $(\mathcal P(\mathcal X), W_1)$ as in Appendix~\ref{app:thm:muconv}. 

For the first term \eqref{eq:m3first}, let $x^N_t \equiv \{x^{i,N}_t\}_{i\in[N]}$. Then, by the weak LLN argument,
\begin{align}
    &\sup_{\pi, \pi^0} \sup_{f \in \mathcal F} \left| \E \left[ f(x^{0,N}_{t+1}, u^{0,N}_{t+1}, \mu^N_{t+1}) - f(x^{0,N}_{t+1}, u^{0,N}_{t+1}, \hat \mu^N_{t+1}) \right] \right| \nonumber\\
    &\quad \leq \sup_{\pi, \pi^0} \E \left[ \tilde \omega_{\mathcal F}(d_\Sigma(\mu^N_{t+1}, \hat \mu^N_{t+1})) \right] \nonumber\\
    &\quad \leq \sup_{\pi, \pi^0} \tilde \omega_{\mathcal F} \left( \sum_{m=1}^\infty 2^{-m} \E \left[ \left| \mu^N_{t+1}(f_m) - \hat \mu^N_{t+1}(f_m) \right| \right] \right) \nonumber\\
    &\quad \leq \sup_{\pi, \pi^0} \tilde \omega_{\mathcal F} \left(  \sup_{m \geq 1} \E \left[ \E_{\beta_t} \left[ \left| \mu^N_{t+1}(f_m) - \hat \mu^N_{t+1}(f_m) \right| \right] \right] \right) \nonumber\\
    &\quad = \sup_{\pi, \pi^0} \tilde \omega_{\mathcal F} \left( \sup_{m \geq 1} \E \left[ \E_{\beta_t} \left[ \left| \frac 1 N \sum_{i=1}^N \left( f_m(x^{i,N}_{t+1}) - \E_{\beta_t} \left[ f_m(x^{i,N}_{t+1}) \right] \right) \right| \right] \right] \right) \nonumber\\
    &\quad \leq \sup_{\pi, \pi^0} \tilde \omega_{\mathcal F} \left( \sup_{m \geq 1} \E \left[ \E_{\beta_t} \left[ \left| \frac 1 N \sum_{i=1}^N \left( f_m(x^{i,N}_{t+1}) - \E_{\beta_t} \left[ f_m(x^{i,N}_{t+1}) \right] \right) \right|^2 \right] \right]^{1/2} \right) \nonumber\\
    &\quad = \sup_{\pi, \pi^0} \tilde \omega_{\mathcal F} \left( \sup_{m \geq 1} \left( \frac{1}{N^2} \sum_{i=1}^N \E \left[ \E_{\beta_t} \left[ \left( f_m(x^{i,N}_{t+1}) - \E_{\beta_t} \left[ f_m(x^{i,N}_{t+1}) \right] \right)^2 \right] \right] \right)^{1/2} \right) \nonumber\\
    &\quad \leq \tilde \omega_{\mathcal F} \left( \frac{2}{\sqrt N} \right) \to 0 \label{eq:m3dconv}
\end{align}
for $\beta_t \coloneqq (x^{0,N}_{t}, u^{0,N}_{t}, x^N_t)$ by bounding $|f_m| \leq 1$, as the cross-terms disappear.

For the second term \eqref{eq:m3second}, by noting $\hat \mu^N_{t+1} = T(x^{0,N}_{t}, u^{0,N}_{t}, \mu^N_t, \mu^N_t \otimes \pi_t(x^{0,N}_{t}, \mu^N_t))$, we have
\begin{align}
    &\sup_{\pi, \pi^0} \sup_{f \in \mathcal F} \left| \E \left[ f(x^{0,N}_{t+1}, u^{0,N}_{t+1}, \hat \mu^N_{t+1}) - f(x^0_{t+1}, u^0_{t+1}, \mu_{t+1}) \right] \right| \nonumber \\
    &\quad = \sup_{\pi, \pi^0} \sup_{f \in \mathcal F} \left| \E \left[ \iint f(x', u', \hat \mu^N_{t+1}) \pi^0_t(\mathrm du' \mid x', \mu^N_{t+1}) p^0(\mathrm dx' \mid x^{0,N}_{t}, u^{0,N}_{t}, \mu^N_t) \nonumber
    \right.\right.\\&\hspace{3cm}\left.\left. 
    - \iint f(x', u', \mu_{t+1}) \pi^0_t(\mathrm du' \mid x', \mu_{t+1}) p^0(\mathrm dx' \mid x^0_{t}, u^{0}_{t}, \mu_t) \right] \right| \nonumber \\
    &\quad \leq \sup_{\pi, \pi^0} \sup_{f \in \mathcal F} \E \left[ \sup_{x'} \left| \int f(x', u', \hat \mu^N_{t+1}) (\pi^0_t(\mathrm du' \mid x', \mu^N_{t+1}) - \pi^0_t(\mathrm du' \mid x', \hat \mu^N_{t+1})) \right| \right] \label{eq:m3q2} \\
    &\qquad + \sup_{\pi, \pi^0} \sup_{g \in \mathcal G} \left| \E \left[ g(x^{0,N}_{t}, u^{0,N}_{t}, \mu^N_t) - g(x^0_{t}, u^0_{t}, \mu_t) \right] \right| \label{eq:m3q3}
\end{align}
and analyze each term separately, where we defined the function $g \colon \mathcal X^0 \times \mathcal U^0 \times \mathcal P(\mathcal X)$ as
\begin{align*}
    g(x^0, u^0, \mu) \coloneqq \iint f(x', u', T^*) \pi^0_t(\mathrm du' \mid x', T^*) p^0(\mathrm dx' \mid x^0, u^0, \mu)
\end{align*}
from the class $\mathcal G$ of such functions for any policies $\pi, \pi^0$, where $T^* \coloneqq T(x^0, u^0, \mu, \mu \otimes \pi_t(x^0, \mu))$.

For \eqref{eq:m3q2}, defining a modulus of continuity $\tilde \omega_{\Pi^0}$ for $\Pi^0$ as for $\mathcal F$, we have
\begin{align*}
    &\sup_{\pi, \pi^0} \sup_{f \in \mathcal F} \E \left[ \sup_{x'} \left| \int f(x', u', \hat \mu^N_{t+1}) (\pi^0_t(\mathrm du' \mid x', \mu^N_{t+1}) - \pi^0_t(\mathrm du' \mid x', \hat \mu^N_{t+1})) \right| \right] \\
    &\quad \leq \sup_{\pi, \pi^0} \E \left[ L_{\mathcal F} \sup_{x'} W_1(\pi^0_t(\cdot \mid x', \mu^N_{t+1}), \pi^0_t(\cdot \mid x', \hat \mu^N_{t+1})) \right] \\
    &\quad \leq \sup_{\pi, \pi^0} \E \left[ L_{\mathcal F} \tilde \omega_{\Pi^0}(d_\Sigma(\mu^N_{t+1}, \hat \mu^N_{t+1})) \right] \leq L_{\mathcal F} \tilde \omega_{\Pi^0} \left( \frac{2}{\sqrt N} \right) \to 0.
\end{align*}

Lastly, for \eqref{eq:m3q3}, we first note that the class $\mathcal G$ of functions is equi-Lipschitz. 

\begin{lemma} \label{lem:app:Tlip}
Under Assumptions~\ref{ass:m3pcont} and \ref{ass:m3picont}, the map $(x^0, u^0, \mu) \mapsto T(x^0, u^0, \mu, \mu \otimes \pi_t(x^0, \mu))$ is Lipschitz with constant $L_T \coloneqq (2L_\Pi + 1) \cdot (L_p + (L_p + 1) L_\Pi + (L_p + L_\Pi + 1))$.
\end{lemma}

\begin{lemma} \label{lem:app:glip}
Under Assumptions~\ref{ass:m3pcont} and \ref{ass:m3picont}, for any equi-Lipschitz $\mathcal F$ with constant $L_{\mathcal F}$, the function class $\mathcal G$ is equi-Lipschitz with constant $L_{\mathcal G} \coloneqq (L_{\mathcal F} L_T + L_{\mathcal F} L_{\Pi^0} L_T + L_{\mathcal F} L_{\Pi} L_{p^0})$.
\end{lemma}

Therefore, for \eqref{eq:m3q3}, we have
\begin{align*}
    \sup_{\pi, \pi^0} \sup_{g \in \mathcal G} \left| \E \left[ g(x^{0,N}_{t}, u^{0,N}_{t}, \mu^N_t) - g(x^0_{t}, u^0_{t}, \mu_t) \right] \right| \to 0
\end{align*}
by the induction assumption over the class $\mathcal G$ of equi-Lipschitz functions, completing the proof by induction. The existence of independent optimal $\pi$, $\pi^0$ follows from Remark~\ref{remark:joint}. This completes the proof.

For finite minor states, we can quantify the convergence rate more precisely as $\mathcal O(1/\sqrt N)$, since the two metrizations $d_\Sigma$ and $W_1$ are then Lipschitz equivalent and the above moduli of continuity simply become a multiplication with the Lipschitz constant, so for convenience we simply use the $L_1$ distance. The convergence in the first term \eqref{eq:m3first} is immediate by the weak LLN
\begin{align*}
    &\sup_{\pi, \pi^0} \sup_{f \in \mathcal F} \left| \E \left[ f(x^{0,N}_{t+1}, u^{0,N}_{t+1}, \mu^N_{t+1}) - f(x^{0,N}_{t+1}, u^{0,N}_{t+1}, \hat \mu^N_{t+1}) \right] \right| \\
    &\quad \leq \sup_{\pi, \pi^0} L_f \E \left[ \sum_{x \in \mathcal X} \left| \mu^N_{t+1}(x) - \hat \mu^N_{t+1}(x) \right| \right] \\
    &\quad = \sup_{\pi, \pi^0} L_f \sum_{x \in \mathcal X} \E \left[ \E \left[ \left| \frac 1 N \sum_{i=1}^N \mathbf 1_x(x^{i,N}_{t+1}) - \E \left[ \frac 1 N \sum_{i=1}^N \mathbf 1_x(x^{i,N}_{t+1}) \innermid x^{0,N}_{t}, u^{0,N}_{t}, \mu^N_t \right] \right| \innermid x^{0,N}_{t}, u^{0,N}_{t}, \mu^N_t \right] \right] \\
    &\quad \leq L_f |\mathcal X| \sqrt{\frac{4}{N}},
\end{align*}
and for the second term \eqref{eq:m3second} we again use the induction assumption, completing the proof.
\end{proof}

\section{Proof of Lemma~\ref{lem:app:Tlip}}
\begin{proof}
First note Lipschitz continuity of $(x^0, \mu) \mapsto \mu \otimes \pi_t(x^0, \mu)$ as in Appendix~\ref{app:thm:muconv}, as for any $(x^0_*, \mu_*), (x^0, \mu) \in \mathcal X^0 \times \mathcal P(\mathcal X)$, then
\begin{align*}
    &\sup_{\pi \in \Pi} W_1(\mu_* \otimes \pi_t(x^0_*, \mu_*), \mu \otimes \pi_t(x^0, \mu)) \\
    &\quad = \sup_{\pi \in \Pi} \sup_{\lVert f' \rVert_{\mathrm{Lip}} \leq 1} \left| \int f' \, \mathrm d(\mu_* \otimes \pi_t(x^0_*, \mu_*) - \mu \otimes \pi_t(x^0, \mu)) \right| \\
    &\quad \leq \sup_{\pi \in \Pi} \sup_{\lVert f' \rVert_{\mathrm{Lip}} \leq 1} \left| \iint f'(x, u) (\pi_t(\mathrm du \mid x, x^0_*, \mu_*) - \pi_t(\mathrm du \mid x, x^0, \mu)) \mu_*(\mathrm dx) \right| \\
    &\qquad + \sup_{\pi \in \Pi} \sup_{\lVert f' \rVert_{\mathrm{Lip}} \leq 1} \left| \iint f'(x, u) \pi_t(\mathrm du \mid x, x^0, \mu) (\mu_*(\mathrm dx) - \mu(\mathrm dx)) \right|
\end{align*}
where for the first term
\begin{align*}
    &\sup_{\pi \in \Pi} \sup_{\lVert f' \rVert_{\mathrm{Lip}} \leq 1} \left| \iint f'(x, u) (\pi_t(\mathrm du \mid x, x^0_*, \mu_*) - \pi_t(\mathrm du \mid x, x^0, \mu)) \mu_*(\mathrm dx) \right| \\
    &\quad \leq \sup_{\pi \in \Pi} \sup_{\lVert f' \rVert_{\mathrm{Lip}} \leq 1} \int \left| \int f'(x, u) (\pi_t(\mathrm du \mid x, x^0_*, \mu_*) - \pi_t(\mathrm du \mid x, x^0, \mu)) \right| \mu_*(\mathrm dx) \\
    &\quad \leq \sup_{\pi \in \Pi} \sup_{\lVert f' \rVert_{\mathrm{Lip}} \leq 1} \sup_{x \in \mathcal X} \left| \int f'(x, u) (\pi_t(\mathrm du \mid x, x^0_*, \mu_*) - \pi_t(\mathrm du \mid x, x^0, \mu)) \right| \\
    &\quad = \sup_{\pi \in \Pi} \sup_{x \in \mathcal X} W_1(\pi_t(\cdot \mid x, x^0_*, \mu_*), \pi_t(\cdot \mid x, x^0, \mu)) \\
    &\quad \leq L_\Pi d((x^0_*, \mu_*), (x^0, \mu))
\end{align*}
by Assumption~\ref{ass:m3picont}, and similarly for the second by noting $1$-Lipschitzness of $x \mapsto \int \frac {f'(x, u)} {L_\Pi + 1}  \pi_t(\mathrm du \mid x, x^0, \mu)$, as before in \eqref{eq:lipschitz-fpi}, and therefore again
\begin{align*}
    &\sup_{\pi \in \Pi} \sup_{\lVert f' \rVert_{\mathrm{Lip}} \leq 1} \left| \iint f'(x, u) \pi_t(\mathrm du \mid x, x^0, \mu) (\mu_*(\mathrm dx) - \mu(\mathrm dx)) \right| \\
    &\quad = \sup_{\pi \in \Pi} \sup_{\lVert f' \rVert_{\mathrm{Lip}} \leq 1} (L_\Pi + 1) \left| \iint \frac{f'(x, u)}{L_\Pi + 1}  \pi_t(\mathrm du \mid x, x^0, \mu) (\mu_*(\mathrm dx) - \mu(\mathrm dx)) \right| \\
    &\quad \leq (L_\Pi + 1) W_1(\mu_*, \mu).
\end{align*}
Hence, the map $(x^0, u^0, \mu) \mapsto \mu \otimes \pi_t(x^0, \mu)$ is Lipschitz with constant $(2L_\Pi + 1)$. 

As a result, the entire map $(x^0, u^0, \mu) \mapsto T(x^0, u^0, \mu, \mu \otimes \pi_t(x^0, \mu)$ is Lipschitz, since for any 
\begin{align*}
    &W_1(T(x^0_*, u^0_*, \mu_*, \mu_* \otimes \pi_t(x^0_*, \mu_*)), T(x^0, u^0, \mu, \mu \otimes \pi_t(x^0, \mu)) \\
    &\quad = \sup_{\lVert f' \rVert_{\mathrm{Lip}} \leq 1} \left| \iiint f'(x') p(\mathrm dx' \mid x, u, x^0_*, u^0_*, \mu_*) \pi_t(\mathrm du \mid x, x^0_*, \mu_*) \mu_*(\mathrm dx)
    \right.\nonumber\\&\hspace{2.5cm}\left.
    - \iiint f'(x') p(\mathrm dx' \mid x, u, x^0, u^0, \mu) \pi_t(\mathrm du \mid x, x^0, \mu) \mu(\mathrm dx) \right| \\
    &\quad \leq \sup_{\lVert f' \rVert_{\mathrm{Lip}} \leq 1} \sup_{(x,u) \in \mathcal X \times \mathcal U} \left| \int f'(x') (p(\mathrm dx' \mid x, u, x^0_*, u^0_*, \mu_*) - p(\mathrm dx' \mid x, u, x^0, u^0, \mu)) \right| \\
    &\qquad + \sup_{\lVert f' \rVert_{\mathrm{Lip}} \leq 1} \sup_{x \in \mathcal X} \left| \iint f'(x') p(\mathrm dx' \mid x, u, x^0, u^0, \mu) (\pi_t(\mathrm du \mid x, x^0_*, \mu_*) - \pi_t(\mathrm du \mid x, x^0, \mu)) \right| \\
    &\qquad + \sup_{\lVert f' \rVert_{\mathrm{Lip}} \leq 1} \left| \iiint f'(x') p(\mathrm dx' \mid x, u, x^0, u^0, \mu) \pi_t(\mathrm du \mid x, x^0, \mu) (\mu_*(\mathrm dx) - \mu(\mathrm dx)) \right| \\
    &\quad \leq \sup_{(x,u) \in \mathcal X \times \mathcal U} W_1(p(\cdot \mid x, u, x^0_*, u^0_*, \mu_*), p(\cdot \mid x, u, x^0, u^0, \mu)) \\
    &\qquad + \sup_{x \in \mathcal X} (L_p + 1) W_1(\pi_t(\cdot \mid x, x^0_*, \mu_*), \pi_t(\cdot \mid x, x^0, \mu)) \\
    &\qquad + \sup_{(x,u) \in \mathcal X \times \mathcal U} (L_p + L_\Pi + 1) W_1(\mu_*, \mu) \\
    &\quad \leq \underbrace{(L_p + (L_p + 1) L_\Pi + (L_p + L_\Pi + 1))}_{L_*} d((x^0_*, u^0_*, \mu_*), (x^0, u^0, \mu))
\end{align*}
with Lipschitz constant $L_T \coloneqq (2L_\Pi + 1) \cdot L_*$ from Assumptions~\ref{ass:m3pcont} and \ref{ass:m3picont}, using the same argument as in \eqref{eq:lipschitz-fpi}. 
\end{proof}

\section{Proof of Lemma~\ref{lem:app:glip}}
\begin{proof}
For any $g \in \mathcal G$, for any $(x^0_*, u^0_*, \mu_*), (x^0, u^0, \mu) \in \mathcal X^0 \times \mathcal U^0 \times \mathcal P(\mathcal X)$, let $T_* \coloneqq T(x^0_*, u^0_*, \mu_*, \mu_* \otimes \pi_t(x^0_*, \mu_*))$ and $T^* \coloneqq T(x^0, u^0, \mu, \mu \otimes \pi_t(x^0, \mu))$ for brevity. We have
\begin{align}
    &\left| g(x^0_*, u^0_*, \mu_*) - g(x^0, u^0, \mu) \right| \nonumber\\
    &\quad = \left| \iint f(x', u', T_*) \pi^0_t(\mathrm du' \mid x', T_*) p^0(\mathrm dx' \mid x^0_*, u^0_*, \mu_*)
    \right.\nonumber\\&\hspace{1.5cm}\left.
    - \iint f(x', u', T^*) \pi^0_t(\mathrm du' \mid x', T^*) p^0(\mathrm dx' \mid x^0, u^0, \mu) \right| \nonumber\\
    &\quad \leq \sup_{x', u'} \left| f(x', u', T_*) - f(x', u', T^*) \right| \label{eq:g1}\\
    &\qquad + \sup_{x'} \left| \int f(x', u', T^*) 
    (\pi^0_t(\mathrm du' \mid x', T_*) - \pi^0_t(\mathrm du' \mid x', T^*)) \right| \label{eq:g2}\\
    &\qquad + \left| \iint f(x', u', T^*) \pi^0_t(\mathrm du' \mid x', T^*)
    (p^0(\mathrm dx' \mid x^0_*, u^0_*, \mu_*) - p^0(\mathrm dx' \mid x^0, u^0, \mu)) \right| \label{eq:g3}.
\end{align}

By Lemma~\ref{lem:app:Tlip}, for \eqref{eq:g1} we obtain
\begin{align*}
    &\sup_{x', u'} \left| f(x', u', T(x^0_*, u^0_*, \mu_*, \mu_* \otimes \pi_t(x^0_*, \mu_*))) - f(x', u', T(x^0, u^0, \mu, \mu \otimes \pi_t(x^0, \mu))) \right| \\
    &\quad \leq L_{\mathcal F} L_T d((x^0_*, u^0_*, \mu_*), (x^0, u^0, \mu)).
\end{align*}

Similarly for \eqref{eq:g2}, by Assumption~\ref{ass:m3picont} we analogously have 
\begin{align*}
    &\sup_{x'} \left| \int f(x', u', T(x^0, u^0, \mu, \mu \otimes \pi_t(x^0, \mu))) 
    \right.\nonumber\\&\hspace{0.5cm}\left.
    (\pi^0_t(\mathrm du' \mid x', T(x^0_*, u^0_*, \mu_*, \mu_* \otimes \pi_t(x^0_*, \mu_*))) - \pi^0_t(\mathrm du' \mid x', T(x^0, u^0, \mu, \mu \otimes \pi_t(x^0, \mu)))) \right| \\
    &\quad \leq L_{\mathcal F} W_1(\pi^0_t(\cdot \mid x', T(x^0_*, u^0_*, \mu_*, \mu_* \otimes \pi_t(x^0_*, \mu_*))), \pi^0_t(\cdot' \mid x', T(x^0, u^0, \mu, \mu \otimes \pi_t(x^0, \mu))) \\
    &\quad \leq L_{\mathcal F} L_{\Pi^0} L_T d((x^0_*, u^0_*, \mu_*), (x^0, u^0, \mu)).
\end{align*}

Lastly, for \eqref{eq:g3}, as before in \eqref{eq:lipschitz-fpi}, by Assumption~\ref{ass:m3pcont} and \ref{ass:m3picont} we have again
\begin{align*}
    &\left| \iint f(x', u', T(x^0, u^0, \mu, \mu \otimes \pi_t(x^0, \mu))) \pi^0_t(\mathrm du' \mid x', T(x^0, u^0, \mu, \mu \otimes \pi_t(x^0, \mu)))
    \right.\nonumber\\&\hspace{2.5cm}\left.
    (p^0(\mathrm dx' \mid x^0_*, u^0_*, \mu_*) - p^0(\mathrm dx' \mid x^0, u^0, \mu)) \right| \\
    &\quad \leq L_{\mathcal F} L_{\Pi} W_1(p^0(\cdot \mid x^0_*, u^0_*, \mu_*), p^0(\cdot \mid x^0, u^0, \mu)) \\
    &\quad \leq L_{\mathcal F} L_{\Pi} L_{p^0} d((x^0_*, u^0_*, \mu_*), (x^0, u^0, \mu)).
\end{align*}

Therefore, $\mathcal G$ is equi-Lipschitz with Lipschitz constant $(L_{\mathcal F} L_T + L_{\mathcal F} L_{\Pi^0} L_T + L_{\mathcal F} L_{\Pi} L_{p^0})$.
\end{proof}

\section{Proof of Corollary~\ref{coro:m3epsopt}} \label{app:coro:m3epsopt}
\begin{proof}
As in Lemma~\ref{lem:Jconv}, for any $\varepsilon > 0$, choose time $T \in \mathbb N$ such that 
\begin{align*}
    \sum_{t=T}^{\infty} \gamma^t \left| \E \left[ r(x^{0,N}_t, u^{0,N}_{t}, \mu_t^N) - r(x^0_t, u^0_{t}, \mu_t) \right] \right| \leq \frac{\gamma^T}{1-\gamma} \max_\mu 2 |r(\mu)| < \frac \varepsilon 2.
\end{align*}
By Theorem~\ref{thm:m3muconv}, 
\begin{align*}
    \sum_{t=0}^{T-1} \gamma^t \left| \E \left[ r(x^{0,N}_t, u^{0,N}_{t}, \mu_t^N) - r(x^0_t, u^0_{t}, \mu_t) \right] \right| < \frac \varepsilon 2
\end{align*}
for sufficiently large $N$. Therefore, $\sup_{(\pi, \pi^0) \in \Pi \times \Pi^0} \left| J^N(\pi, \pi^0) - J(\Phi^{-1}(\pi), \pi^0) \right| \to 0$. 

As a result, we have
\begin{align*}
    J^N(\Phi(\hat \pi^*), \pi^{0*}) - \sup_{(\pi, \pi^0) \in \Pi \times \Pi^0} J^N(\pi, \pi^0)
    &= \inf_{(\pi, \pi^0) \in \Pi \times \Pi^0} (J^N(\Phi(\hat \pi^*), \pi^{0*}) - J^N(\pi, \pi^0)) \\
    &\geq \inf_{(\pi, \pi^0) \in \Pi \times \Pi^0} (J^N(\Phi(\hat \pi^*), \pi^{0*}) - J(\hat \pi^*, \pi^{0*})) \\
    &\quad + \inf_{(\pi, \pi^0) \in \Pi \times \Pi^0} (J(\hat \pi^*, \pi^{0*}) - J(\pi, \pi^0)) \\
    &\quad + \inf_{(\pi, \pi^0) \in \Pi \times \Pi^0} (J(\pi, \pi^0) - J^N(\pi, \pi^0)) \\
    &\geq - \frac \varepsilon 2 + 0 - \frac \varepsilon 2 = - \varepsilon
\end{align*}
for sufficiently large $N$, where the second term is zero by optimality of $(\hat \pi^*, \pi^{0*})$ in the M3FC problem.
\end{proof}

\section{Proof of Theorem~\ref{thm:pg}} \label{app:pg}
First, for completeness we give the finite M3FC system equations under the assumed Lipschitz parametrization for joint stationary M3FMARL policies\footnote{Note that deterministic joint policies $\tilde \pi^\theta$ (e.g. at convergence, or if using deterministic policy gradients \citep{silver2014deterministic}) are equivalent to using separate deterministic minor and major policies in \eqref{eq:m3mdp}, see also Remark~\ref{remark:joint}.} $\tilde \pi^\theta$ used during centralized training with correlated minor agent actions, as
\begin{align*}
    u^{0,N}_{t}, \xi^N_t &\sim \tilde \pi^\theta(u^{0,N}_{t}, \xi^N_t \mid x^{0,N}_t, \mu_t^N), \quad
    \pi'^N_t = \Gamma(\xi^N_t), \quad
    u^{i,N}_t \sim \pi'^N_t(u^{i,N}_t \mid x^{i,N}_t), \\
    x^{i,N}_{t+1} &\sim p(x^{i,N}_{t+1} \mid x^{i,N}_t, u^{i,N}_t, x^{0,N}_t, u^{0,N}_{t}, \mu_t^N), \quad
    x^{0,N}_{t+1} \sim p^0(x^{0,N}_{t+1} \mid x^{0,N}_t, u^{0,N}_{t}, \mu_t^N),
\end{align*}
as well as the limiting M3FC MDP under such parametrization as
\begin{align*}
    u^0_{t}, \xi_t &\sim \tilde \pi^\theta(u^0_{t}, \xi_t \mid x^0_t, \mu_t), \quad
    \pi'_t = \Gamma(\xi_t), \quad
    h_t = \mu_t \otimes \pi'_t, \\
    \mu_{t+1} &= T(x^0_t, u^0_{t}, \mu_t, h_t), \quad
    x^0_{t+1} \sim p^0(x^0_{t+1} \mid x^0_t, u^0_{t}, \mu_t).
\end{align*}

Then, by \citet{sutton1999policy}, the exact policy gradient for the limiting M3FC MDP is given as
\begin{align*}
    \nabla_\theta J(\tilde \pi^\theta) = \sum_{t=T}^{\infty} \gamma^t \E \left[ Q^\theta(x^0_t, \mu_t, u^0_t, \xi_t) \nabla_\theta \log \tilde \pi^\theta(u^0_t, \xi_t \mid x^0_t, \mu_t) \right]
\end{align*}
under the action-value function
\begin{align*}
    Q^\theta(x^0, \mu, u^0, \xi) = \E \left[ \sum_{t=0}^{\infty} \gamma^t r(x^0_t, u^0_t, \mu_t) \innermid x^0_0 = x^0, \mu_0 = \mu, u^0_0 = u^0, \xi_0 = \xi \right],
\end{align*}
while the approximation for the policy gradient on the finite M3FC system is given instead by
\begin{align*}
    \widehat{\nabla_\theta J}(\tilde \pi^\theta) = \sum_{t=T}^{\infty} \gamma^t \E \left[ \widehat Q^\theta(x^{0,N}_t, \mu^N_t, u^{0,N}_t, \xi^N_t) \nabla_\theta \log \tilde \pi^\theta(u^{0,N}_t, \xi^N_t \innermid x^{0,N}_t, \mu^N_t) \right]
\end{align*}
and the finite-agent action-values
\begin{align*}
    \widehat Q^\theta(x^0, \mu, u^0, \xi) = \E \left[ \sum_{t=0}^{\infty} \gamma^t r(x^{0,N}_t, u^{0,N}_t, \mu^N_t) \innermid x^{0,N}_0 = x^0, \mu_0 = \mu, u^{0,N}_0 = u^0, \xi^N_0 = \xi \right],
\end{align*}
which are obtained, e.g., by on-policy samples and using critic estimates. Note that here, the conditional expectations are given by redefining the systems \eqref{eq:m3mdp} and \eqref{eq:m3fc} with the values conditioned upon. 

We then show that the approximation of the policy gradient is good for large systems, i.e. 
\begin{align}
    \left\Vert \widehat{\nabla_\theta J}(\tilde \pi^\theta) - \nabla_\theta J(\hat \pi^\theta) \right\Vert \to 0
\end{align}
as $N \to \infty$, uniformly over all current policy parameters $\theta$.

\begin{proof}[Proof of Theorem~\ref{thm:pg}]
    We use the following lemmas in the proof of Theorem~\ref{thm:pg}, for which the proofs are given below.

    \begin{proposition} \label{prop:Nconv}
        Propagation of chaos holds for the M3FC systems with parameterized actions as in Theorem~\ref{thm:m3muconv}, i.e. under Assumptions~\ref{ass:m3pcont}, \ref{ass:m3picont} and \ref{ass:pg}, for any equi-Lipschitz family $\mathcal F$, at all times $t \in \mathbb N$ uniformly,
        \begin{equation} \label{eq:prop:Nconv}
            \sup_{f, \pi, \pi^0} \left| \E \left[ f(x^{0,N}_t, u^{0,N}_{t}, \mu_t^N) - f(x^0_t, u^0_{t}, \mu_t) \right] \right| \to 0.
        \end{equation}
    \end{proposition}

    \begin{proposition} \label{prop:Qconv}
        Under Assumptions~\ref{ass:m3pcont} and \ref{ass:m3picont}, the approximate action-values converge uniformly, $\widehat Q^\theta \to Q^\theta$ as $N \to \infty$.
    \end{proposition}

    As a result, we obtain
    \begin{align*}
        &\left\Vert \widehat{\nabla_\theta J}(\tilde \pi^\theta) - \nabla_\theta J(\hat \pi^\theta) \right\Vert \\
        &= \left\Vert \sum_{t=0}^{\infty} \gamma^t \E \left[ \widehat Q^\theta(x^{0,N}_t, \mu^N_t, u^{0,N}_t, \xi^N_t) \nabla_\theta \log \tilde \pi^\theta(u^{0,N}_t, \xi^N_t \mid x^{0,N}_t, \mu^N_t) - Q^\theta(x^0_t, \mu_t, u^0_t, \xi_t) \nabla_\theta \log \tilde \pi^\theta(u^0_t, \xi_t \mid x^0_t, \mu_t) \right] \right\Vert \\
        &\leq \left\Vert \sum_{t=0}^{\infty} \gamma^t \E \left[ \left( \widehat Q^\theta(x^{0,N}_t, \mu^N_t, u^{0,N}_t, \xi^N_t) - Q^\theta(x^{0,N}_t, \mu^N_t, u^{0,N}_t, \xi^N_t) \right) \nabla_\theta \log \tilde \pi^\theta(u^{0,N}_t, \xi^N_t \mid x^{0,N}_t, \mu^N_t) \right] \right\Vert \\
        &+ \left\Vert \sum_{t=T}^{\infty} \gamma^t \E \left[ Q^\theta(x^{0,N}_t, \mu^N_t, u^{0,N}_t, \xi^N_t) \nabla_\theta \log \tilde \pi^\theta(u^{0,N}_t, \xi^N_t \mid x^{0,N}_t, \mu^N_t) - Q^\theta(x^0_t, \mu_t, u^0_t, \xi_t) \nabla_\theta \log \tilde \pi^\theta(u^0_t, \xi_t \mid x^0_t, \mu_t) \right] \right\Vert \\
        &+ \left\Vert \sum_{t=0}^{T-1} \gamma^t \E \left[ Q^\theta(x^{0,N}_t, \mu^N_t, u^{0,N}_t, \xi^N_t) \nabla_\theta \log \tilde \pi^\theta(u^{0,N}_t, \xi^N_t \mid x^{0,N}_t, \mu^N_t) - Q^\theta(x^0_t, \mu_t, u^0_t, \xi_t) \nabla_\theta \log \tilde \pi^\theta(u^0_t, \xi_t \mid x^0_t, \mu_t) \right] \right\Vert
    \end{align*}
    for any $T$, such that the first term disappears by Assumption~\ref{ass:pg} uniformly bounding $\nabla_\theta \log \tilde \pi^\theta$ and Proposition~\ref{prop:Qconv}. Note that we bounded $\nabla_\theta \log \tilde \pi^\theta$ here, but we can also assume bounded gradients $\nabla_\theta \tilde \pi^\theta$ instead, e.g. \eqref{eq:baseint}.

    For the second term, we similarly uniformly bound $\nabla_\theta \log \tilde \pi^\theta$ by Assumption~\ref{ass:pg} and $Q$ by Assumption~\ref{ass:m3pcont}, then choose $T$ sufficiently large.

    Finally, for the last term, we note that we can write the difference as 
    \begin{align*}
        &\left\Vert \sum_{t=0}^{T-1} \gamma^t \E \left[ Q^\theta(x^{0,N}_t, \mu^N_t, u^{0,N}_t, \xi^N_t) \nabla_\theta \log \tilde \pi^\theta(u^{0,N}_t, \xi^N_t \mid x^{0,N}_t, \mu^N_t) - Q^\theta(x^0_t, \mu_t, u^0_t, \xi_t) \nabla_\theta \log \tilde \pi^\theta(u^0_t, \xi_t \mid x^0_t, \mu_t) \right] \right\Vert \\
        &= \left\Vert \sum_{t=0}^{T-1} \gamma^t \E \left[ \sum_{t'=0}^{\infty} \gamma^t \E \left[ r(x^{0\prime}_{t'}, u^{0\prime}_{t'}, \mu^{\prime}_{t'}) \innermid x^{0\prime}_0 = x^{0,N}_t, \mu'_0 = \mu^N_t, u^{0\prime}_0 = u^{0,N}_t, \xi^{\prime}_0 = \xi^N_t \right] \nabla_\theta \log \tilde \pi^\theta(u^{0,N}_t, \xi^N_t \mid x^{0,N}_t, \mu^N_t) 
        \right.\right.\nonumber\\&\hspace{2cm}\left.\left.
        - \sum_{t=0}^{\infty} \gamma^t \E \left[ r(x^{0\prime}_{t'}, u^{0\prime}_{t'}, \mu'_{t'}) \innermid x^{0\prime}_0 = x^0_t, \mu'_0 = \mu_t, u^{0\prime}_0 = u^0_t, \xi'_0 = \xi_t \right] \nabla_\theta \log \tilde \pi^\theta(u^0_t, \xi_t \mid x^0_t, \mu_t) \right] \right\Vert \\
        &\leq \left\Vert \sum_{t=0}^{T-1} \gamma^t \E \left[ \sum_{t'=T'}^{\infty} \gamma^t \E \left[ r(x^{0\prime}_{t'}, u^{0\prime}_{t'}, \mu^{\prime}_{t'}) \innermid x^{0\prime}_0 = x^{0,N}_t, \mu'_0 = \mu^N_t, u^{0\prime}_0 = u^{0,N}_t, \xi^{\prime}_0 = \xi^N_t \right] \nabla_\theta \log \tilde \pi^\theta(u^{0,N}_t, \xi^N_t \mid x^{0,N}_t, \mu^N_t) 
        \right.\right.\nonumber\\&\hspace{2cm}\left.\left.
        - \sum_{t=T'}^{\infty} \gamma^t \E \left[ r(x^{0\prime}_{t'}, u^{0\prime}_{t'}, \mu'_{t'}) \innermid x^{0\prime}_0 = x^0_t, \mu'_0 = \mu_t, u^{0\prime}_0 = u^0_t, \xi'_0 = \xi_t \right] \nabla_\theta \log \tilde \pi^\theta(u^0_t, \xi_t \mid x^0_t, \mu_t) \right] \right\Vert \\
        &+ \left\Vert \sum_{t=0}^{T-1} \gamma^t \E \left[ \sum_{t'=0}^{T'-1} \gamma^t \E \left[ r(x^{0\prime}_{t'}, u^{0\prime}_{t'}, \mu^{\prime}_{t'}) \innermid x^{0\prime}_0 = x^{0,N}_t, \mu'_0 = \mu^N_t, u^{0\prime}_0 = u^{0,N}_t, \xi^{\prime}_0 = \xi^N_t \right] \nabla_\theta \log \tilde \pi^\theta(u^{0,N}_t, \xi^N_t \mid x^{0,N}_t, \mu^N_t) 
        \right.\right.\nonumber\\&\hspace{2cm}\left.\left.
        - \sum_{t=0}^{T'-1} \gamma^t \E \left[ r(x^{0\prime}_{t'}, u^{0\prime}_{t'}, \mu'_{t'}) \innermid x^{0\prime}_0 = x^0_t, \mu'_0 = \mu_t, u^{0\prime}_0 = u^0_t, \xi'_0 = \xi_t \right] \nabla_\theta \log \tilde \pi^\theta(u^0_t, \xi_t \mid x^0_t, \mu_t) \right] \right\Vert
    \end{align*}
    where we write the conditional M3FC system and random variables in the inner expectation with a prime, bounding again the former terms by choosing sufficiently large $T'$ and using Assumptions~\ref{ass:m3pcont} and \ref{ass:pg}, while for the latter terms we use Proposition~\ref{prop:Nconv} on the functions
    \begin{align} \label{eq:baseint}
        f(x^0, \mu) = \iint \E \left[ r(x^{0\prime}_{t'}, u^{0\prime}_{t'}, \mu^{\prime}_{t'}) \innermid x^{0\prime}_0 = x^0, \mu_0 = \mu, u^{0\prime}_0 = u^0, \xi^{\prime}_0 = \xi \right] \nabla_\theta \tilde \pi^\theta(u^0, \xi \mid x^0, \mu) \mathrm d(u^0, \xi)
    \end{align}
    for all $t'$, which are uniformly Lipschitz by Assumptions~\ref{ass:m3pcont} and \ref{ass:pg}. This completes the proof.
\end{proof}

\section{Proof of Proposition~\ref{prop:Nconv}}
\begin{proof}
The proof is exactly analogous to the proof of Theorem~\ref{thm:m3muconv}, except that instead of using Lipschitz constants of $x^0_t, u^0_{t}, \mu_t, h_t \mapsto T(x^0_t, u^0_{t}, \mu_t, h_t)$, one uses Lipschitz constants of $x^0_t, u^0_{t}, \mu_t, \xi_t \mapsto T(x^0_t, u^0_{t}, \mu_t, \mu_t \otimes \Gamma(\xi_t))$ via the additional Assumption~\ref{ass:pg} on top of Assumptions~\ref{ass:m3pcont} and \ref{ass:m3picont}.
\end{proof}

\section{Proof of Proposition~\ref{prop:Qconv}} \label{app:endp}
\begin{proof}
To show $\widehat Q^\theta \to Q^\theta$ as $N \to \infty$ uniformly, it suffices to prove pointwise convergence due to compact support.

Therefore, fix any $x^0, \mu, u^0, \xi$. The convergence follows as in Corollary~\ref{coro:m3epsopt}, from showing at any time $t$ that
\begin{align*}
    &\sup_{f \in \mathcal F} \left| \E \left[ f(x^0_t, u^0_t, \mu_t) \innermid x^0_0 = x^0, \mu_0 = \mu, u^0_0 = u^0, \xi_0 = \xi \right] 
    \right.\\&\hspace{3cm}\left.
    - \E \left[ f(x^{0,N}_t, u^{0,N}_t, \mu^N_t) \innermid x^{0,N}_0 = x^{0,N}, \mu_0 = \mu, u^{0,N}_0 = u^0, \xi^N_0 = \xi \right] \right| \to 0
\end{align*}
over any equi-Lipschitz family of functions $\mathcal F$, and applying for $f=r$ (using the set $\mathcal F$ of $L_r$-Lipschitz functions) by Assumption~\ref{ass:m3pcont}.

The statement is shown by considering time $t=0$, and then by induction for any $t \geq 1$. At time $t=0$, the statement follows from the weak LLN as in Theorem~\ref{thm:m3muconv}. 
For any subsequent times, we similarly have
\begin{align*}
    &\sup_{f \in \mathcal F} \left| \E \left[ f(x^0_{t+1}, u^0_{t+1}, \mu_{t+1}) \innermid x^0_0 = x^0, \mu_0 = \mu, u^0_0 = u^0, \xi_0 = \xi \right] 
    \right.\\&\hspace{2cm}\left.
    - \E \left[ f(x^{0,N}_{t+1}, u^{0,N}_{t+1}, \mu^N_{t+1}) \innermid x^{0,N}_0 = x^{0,N}, \mu_0 = \mu, u^{0,N}_0 = u^0, \xi^N_0 = \xi \right] \right| \\
    &\quad \leq \sup_{f \in \mathcal F} \left| \E \left[ f(x^0_{t+1}, u^0_{t+1}, \mu_{t+1}) \innermid x^0_0 = x^0, \mu_0 = \mu, u^0_0 = u^0, \xi_0 = \xi \right] 
    \right.\\&\hspace{2cm}\left.
    - \E \left[ f(x^{0,N}_{t+1}, u^{0,N}_{t+1}, T(x^{0,N}_t, u^{0,N}_t, \mu^N_t, \mu^N_t \otimes \Gamma(\xi^N_t))) \innermid x^{0,N}_0 = x^{0,N}, \mu_0 = \mu, u^{0,N}_0 = u^0, \xi^N_0 = \xi \right] \right| \\
    &\qquad + \sup_{f \in \mathcal F} \left| \E \left[ f(x^{0,N}_{t+1}, u^{0,N}_{t+1}, T(x^{0,N}_t, u^{0,N}_t, \mu^N_t, \mu^N_t \otimes \Gamma(\xi^N_t))) \innermid x^{0,N}_0 = x^{0,N}, \mu_0 = \mu, u^{0,N}_0 = u^0, \xi^N_0 = \xi \right]
    \right.\\&\hspace{2cm}\left.
    - \E \left[ f(x^{0,N}_{t+1}, u^{0,N}_{t+1}, \mu^N_{t+1}) \innermid x^{0,N}_0 = x^{0,N}, \mu_0 = \mu, u^{0,N}_0 = u^0, \xi^N_0 = \xi \right] \right|.
\end{align*}
As in Theorem~\ref{thm:m3muconv}, the latter term is bounded by induction assumption, using uniform Lipschitzness of the dynamics, $x^0_t, u^0_{t}, \mu_t, \xi_t \mapsto T(x^0_t, u^0_{t}, \mu_t, \mu_t \otimes \Gamma(\xi_t))$ via Assumptions~\ref{ass:m3picont} and \ref{ass:pg}, while the former term is bounded as usual by the weak LLN. This completes the proof.
\end{proof}

\section{Extended MFC Optimalities} \label{app:moreopt}
Intuitively, in large MF systems governed by dynamics of the form \eqref{eq:m3mdp}, almost all information of the joint state $(x^{0,N}_t, x^{1,N}_t, \ldots, x^{N,N}_t)$ is contained in $(x^{0,N}_t, \mu^N_t)$, while heterogeneous policies should by LLN be replaceable by a shared one. To fully complete the theory of MFC, it is therefore interesting to establish the optimality of the considered MF policies over arbitrary other policies acting on the joint state $(x^{0,N}_t, x^{1,N}_t, \ldots, x^{N,N}_t)$. 

It seems plausible that it would be possible to extend optimality (Corollary~\ref{coro:m3epsopt}) over larger classes of policies in the finite system. In particular, at least for finite state-action spaces, (i) any joint-state policy $\pi(\mathrm du \mid x^{0,N}_t, x^{1,N}_t, \ldots, x^{N,N}_t)$ might in the limit be replaced by an averaged policy $\bar \pi(\mathrm du \mid x^0, \mu) \coloneqq \sum_{x^N \in \mathcal X^N \colon \frac 1 N \sum_i \delta_{x^{i,N}} = \mu} \pi(\mathrm du \mid x^0, x^N)$ under some exchangeability of agents; (ii) any optimal policy $\pi$ outputting joint actions for all agents might be replaced by an independent but identical policy for each agent, as in the limit all information is contained in the joint state-action distribution, any of which may be approximated increasingly closely by LLN; and (iii) heterogeneous policies for each minor agent $\pi^1, \ldots, \pi^N$ might similarly be replaced by some averaged policy $\bar \pi(\pi^1, \ldots, \pi^N)$, averaging the action distributions in any specific state over the proportion of agent likelihoods in that state. 

Showing such results would allow us to conclude that the policy classes $\Pi$ are natural and sufficient in MF systems, including MFC and also the competitive MFGs, as more general or heterogeneous policies will not perform much better. A result related to (iii) has been shown for static cases \citep{sanjari2020optimal, cui2021discrete} and more recently in MFC and its two-team generalizations \citep{guan2024zero}.

\section{Experimental Details} \label{app:exp}
In this section, we give lengthy experimental details that were omitted in the main text.

\begin{table}
    \centering
    \caption{Shared hyperparameter configurations for all algorithms.}
    \vspace{0.3cm}
    \label{tab:hyperparams}
    \begin{tabular}{@{}ccc@{}}
    \toprule
    Symbol     & Name          & Value     \\ \midrule
    $\gamma$ &   Discount factor &  $0.99$\\
    $\lambda$ &   GAE lambda &  $1$\\
    $\beta$ &   KL coefficient & $0.03$ \\
    $\epsilon$ &  Clip parameter & $0.2$ \\
    $l_{r}$ &   Learning rate & $0.00005$ \\
    $B_{\mathrm{len}}$ &  Training batch size &  $24000$ \\
    $b_{\mathrm{len}}$ &  Mini-batch size &  $4000$ \\
    $N_{\mathrm{SGD}}$ &  Gradient steps per training batch & $8$ \\ \bottomrule
    \end{tabular}
\end{table}

\subsection{Problem Details}
In this section, we give details to the problems considered in this work. We omit the superscript $N$ for readability.

\paragraph{2G.}
In the 2G problem, we formally let $\mathcal X = [-2, 2]^2$, $\mathcal U = [-1, 1]^2$, $\mathcal X^0 =\{0, 1, \ldots 49\}$ according to \eqref{eq:msmmdp}. We allow noisy movement of minor agents following the Gaussian law
\begin{align*}
    p(x^i_{t+1} \mid x^i_{t}, u^i_{t}) = \mathcal N \left( x^i_{t+1} \innermid x^i_{t} + v_{\mathrm{max}} \frac{u^i_{t}}{\max(1, \lVert u^i_{t} \rVert_2)}, \mathrm{diag}(\sigma^2, \sigma^2) \right)
\end{align*}
for some maximum speed $v_{\mathrm{max}} = 0.2$, noise covariance $\sigma^2 = 0.03$ and projecting back actions $u$ with norm larger than $1$, with the additional modification that agent positions are clipped back into $\mathcal X$ whenever the agents move out of bounds.

We then consider a time-variant mixture of two Gaussians
\begin{align*}
    \mu^*_t \coloneqq \frac{1 + \cos(2 \pi t / 50)}{2} \mathcal N \left( \mathbf e_1, \mathrm{diag}(\sigma_*^2, \sigma_*^2) \right) + \frac{1 - \cos(2 \pi t / 50)}{2} \mathcal N \left( -\mathbf e_1, \mathrm{diag}(\sigma_*^2, \sigma_*^2) \right)
\end{align*}
for unit vector $\mathbf e_1$ and covariance $\sigma_*^2 = 0.05$, i.e. we have a period of $50$ time steps, and let the major state follow the clock dynamics $p^0(x^0 + 1 \mod 50 \mid x^0, \mu) = 1$. 

The goal of minor agents is to minimize the Wasserstein metric $\hat W_1$ under the squared Euclidean distance,
\begin{align*}
    \hat W_1(\mu, \mu') \coloneqq \inf_{\gamma \in \Gamma(\mu, \mu')} \left\{ \int \lVert x - y \rVert_2^2 \gamma(\mathrm dx, \mathrm dy) \right\}
\end{align*}
defined over all couplings $\Gamma(\mu, \mu')$ with first and second marginals $\mu$, $\mu'$ (which is strictly speaking not a metric but an optimal transportation cost, since the squared Euclidean distance fails the triangle inequality), between their empirical distribution and the desired mixture of Gaussians
\begin{align*}
    r(x^0_t, \mu_t) = - \hat W_1(\mu_t, \mu^*_t)
\end{align*}
which is computed numerically by the empirical distance, sampling $300$ samples from $\mu^*_t$.

The initialization of minor agents is uniform, i.e. $\mu_0 = \mathrm{Unif}(\mathcal X)$, and $x^0_0 = 0$. For sake of simulation, we define the episode length $T=100$ after which a new episode starts.

\paragraph{Formation.}
The Formation problem is an extension of the 2G problem, where instead $\mathcal X^0 = \mathcal X \times \mathcal X$ and $\mathcal U^0 = \mathcal U$, the major agent follows the same dynamics as the minor agents, and movements are noise-free, i.e. $\sigma^2 = 0$. The major agent state $x^0_t = (\hat x^0_t, x^*_t)$ here contains both the major agent position $\hat x^0_t$ and its target position $x^*_t$. The desired minor agent distribution is centered around the major agent
\begin{align*}
    \mu^*_t \coloneqq \mathcal N \left( \hat x^0_t, \mathrm{diag}(\sigma_*^2, \sigma_*^2) \right)
\end{align*}
with covariance $\sigma_*^2 = 0.3$, and is also observed by agents as in 2G via binning. Additionally, the major agent should follow a random target $x^*_t$ following discretized Ornstein-Uhlenbeck dynamics
\begin{align*}
    x^*_{t+1} \sim \mathcal N \left( 0.95 x^*_t, \mathrm{diag}(\sigma_{\mathrm{targ}}^2, \sigma_{\mathrm{targ}}^2) \right)
\end{align*}
with $\sigma_{\mathrm{targ}}^2 = 0.02$. Thus, similar to 2G, the reward function becomes
\begin{align*}
    r(x^0_t, u^0_t, \mu_t) = - \lVert \hat x^0_t - x^*_t \rVert_2 - \hat W_1(\mu_t, \mu^*_t).
\end{align*}

The initialization of agents is uniform, while the target starts around zero, i.e. $\mu_0 = \mathrm{Unif}(\mathcal X)$ and $\mu^0_0 = \mathrm{Unif}(\mathcal X) \otimes \mathcal N \left( 0, \mathrm{diag}(\sigma_{\mathrm{targ}}^2, \sigma_{\mathrm{targ}}^2) \right)$. For sake of simulation, we define the episode length $T=100$ after which a new episode starts.

\paragraph{Beach Bar Process.}
In the discrete beach bar process, we consider a discrete torus $\mathcal X = \{0, 1, \ldots, 4\}^2$, $\mathcal X^0 = \mathcal X \times \mathcal X$ and actions $\mathcal U = \mathcal U^0 = \{(0,0), (-1,0), (0,-1), (1,0), (0,1)\}$ indicating movement in any of the four cardinal directions. The major agent state $x^0_t = (\hat x^0_t, x^*_t)$ here contains both the major agent position $\hat x^0_t$ and its target position $x^*_t$. In other words, the dynamics follow
\begin{align*}
    \hat x^0_{t+1} = \hat x^0_{t} + u^0_{t} \mod (5,5), \quad
    x^i_{t+1} = x^i_{t} + u^i_{t} \mod (5,5).
\end{align*}
The target position follows a random walk on the torus
\begin{align*}
    x^*_{t+1} \sim x^*_t + \epsilon_t \mathrm{Unif}((-1,0), (0,-1), (1,0), (0,1)) \mod (5,5)
\end{align*}
with walking probability $\epsilon_t \sim \mathrm{Bernoulli}(0.2)$, uniformly in any direction.

The costs are then given by the average toroidal distance $d$ (the $L_1$ ``wrap-around'' distance on the torus) between the major agent and its target, the average distance between major and minor agents, and the crowdedness of agents
\begin{align*}
    r(x^0_t, u^0_t, \mu_t) = - 0.5 d(x^0_t, x^*_t) - 2.5 \int d(x, x^0_t) \mu_t(\mathrm dx) - 6.25 \int \mu_t(x) \mu_t(\mathrm dx).
\end{align*}
The initialization of agents is uniform, while the target starts at zero, i.e. $\mu_0 = \mathrm{Unif}(\mathcal X)$ and $\mu^0_0 = \mathrm{Unif}(\mathcal X) \otimes \delta_{(0,0)}$. For sake of simulation, we define the episode length $T=200$ after which a new episode starts.

For the neural network policy, we use a one-hot encoding of major states as input, i.e. the concatenation of two $5$-dimensional one-hot vectors for the major agent position $\hat x^0_t$ and its target position $x^*_t$ respectively.

\paragraph{Foraging.}
In the Foraging problem, we formally define $\mathcal X = [-2, 2]^2 \times [0, 1]$, $\mathcal U = [-1, 1]^2 = \mathcal U^0$ and $\mathcal X^0 = ([-2, 2] \times [-2, -1]) \times \bigcup_{n=0}^5 \left( [-2, 2]^2 \times [0, 1.5] \right)^n$. The minor agent states $x^i_t = (\hat x^i_t, \tilde x^i_t)$ here contain their positions $\hat x^i_t \in [-2, 2]^2$ and encumbrance (or inversely, free cargo space) $\hat x^i_t \in [0, 1]$. Meanwhile, the major agent state $x^0_t = (\hat x^0_t, x^{\mathrm{env}}_t)$ here contains both the major agent position $\hat x^0_t$ restricted to $[-2, 2] \times [-2, -1]$, and the current environment state $x^{\mathrm{env}}_t$. Here, the minor and major agents move as in Formation, though with different maximum velocities for minor agents $v_{\mathrm{max}} = 0.3$ and major agent $v^0_{\mathrm{max}} = 0.1$ respectively. 

An additional environmental state consists of up to $5$ spatially localized foraging areas, which is not observed by the agents. In each time step, $N_t = \mathrm{Pois}(0.2)$ new foraging areas appear, up to a maximum total number of $5$. The location $x^m_t$ of each foraging area $m=1,\ldots,5$ is sampled uniformly randomly from $\mathrm{Unif}(\mathcal X)$, while their total initial size $L^m_t$ is sampled from $\mathrm{Unif}([0.5, 1.5])$, making up the environment state $x^{\mathrm{env}}_t = (x^m_t, L^m_t)_m$. At every time step, the foraging areas $m$ are depleted by nearby agents closer than range $0.5$,
\begin{align*}
    L^m_{t+1} &= L^m_t - \Delta L^m(\mu_t), \\
    \Delta L^m(\mu_t) &\coloneqq \min(L^m_{t+1} - L^m_t, \min(0.1, \int (0.5 - \lVert x - x^m_t \rVert_2)^+ \, \mu_t(\mathrm dx))
\end{align*}
where $(\cdot)^+ \coloneqq \max(0, \cdot)$, until they are fully depleted and disappear ($L^m_{t+1} \leq 0$).

Foraging minor agents simulate encumbrance, gaining it from nearby foraging areas and depositing to a nearby major agent, by splitting the foraged amount among all nearby minor agents according to their foraged contribution, and wasting any amount going beyond maximum encumbrance $1$,
\begin{align*}
    \tilde x^i_{t+1} =
    \begin{cases}
        \min(1, \tilde x^i_t + \Delta L^m(\mu_t) \cdot \frac{(0.5 - \lVert x - x^m_t \rVert_2)^+}{\int (0.5 - \lVert x - x^m_t \rVert_2)^+ \, \mu_t(\mathrm dx)}) \quad \text{if} \quad \lVert x^i_t - x^0_t \rVert_2 \geq 0.5, \\
        0 \quad \text{else.}
    \end{cases}
\end{align*}
The reward at each time step is then given by the according total foraged and then deposited amount by the minor agents, where any clipped amount is wasted.

The initialization of agents is uniform, while the environment starts empty, i.e. $\mu_0 = \mathrm{Unif}(\mathcal X)$ and $\mu^0_0 = \mathrm{Unif}(\mathcal X) \otimes \delta_{\emptyset}$. For sake of simulation, we define the episode length $T=200$ after which a new episode starts.

\paragraph{Potential.}
Lastly, in Potential we consider minor agents on a continuous one-dimensional torus $\mathcal X = [-2, 2]$ (where the points $-2$ and $2$ are identified), actions $\mathcal U = [-1, 1]$ and major state $\mathcal X^0 = \mathcal X \times \mathcal X$. The minor agents move as in Foraging (wrapping around the torus instead of clipping), while the major agent follows the gradient of the potential landscape generated by minor agents, with the goal of staying close to its current target. The major agent state $x^0_t = (\hat x^0_t, x^*_t)$ here contains both the major agent position $\hat x^0_t$ and its target position $x^*_t$. For simplicity, here we use a linear repulsive force decreasing from $\frac 1 N$ to $0$ over a range of $1$,
\begin{align*}
    \hat x^0_{t+1} = \hat x^0_t + \frac{1}{20} \sum_{x_{\mathrm{off}} \in \{-4, 0, 4\}} \int (1 - \lVert \hat x^0_t - x + x_{\mathrm{off}} \rVert_2)^+ \frac{\hat x^0_t - x + x_{\mathrm{off}}}{\lVert \hat x^0_t - x + x_{\mathrm{off}} \Vert_2} \mu_t(\mathrm dx) \mod [-2, 2]
\end{align*}
where we let terms $0/0=0$ and use the offset $x_{\mathrm{off}}$ to account for the wrap-around on the torus. 

The target follows the discretized Ornstein-Uhlenbeck process
\begin{align*}
    x^*_{t+1} \sim \mathcal N \left( 0.99 x^*_t, \mathrm{diag}(\sigma_{\mathrm{targ}}^2, \sigma_{\mathrm{targ}}^2) \right)
\end{align*}
with covariance $\sigma_{\mathrm{targ}}^2 = 0.005$, and gives rise to the reward function via the toroidal distance between target and major agent
\begin{align*}
    r(x^0_t, \mu_t) = - d(\hat x^0_t, x^*_t).
\end{align*}
The initialization of agents is uniform, while the target starts around zero, i.e. $\mu_0 = \mathrm{Unif}(\mathcal X)$ and $\mu^0_0 = \mathrm{Unif}(\mathcal X) \otimes \mathcal N \left( 0, \mathrm{diag}(\sigma_{\mathrm{targ}}^2, \sigma_{\mathrm{targ}}^2) \right)$. For sake of simulation, we define the episode length $T=100$ after which a new episode starts. In contrast to $M = 7^2 = 49$ in 2G, Formation and Foraging, here we use $M=7$ bins for the one-dimensional problem.

\begin{figure}
    \centering
    \includegraphics[width=0.99\linewidth]{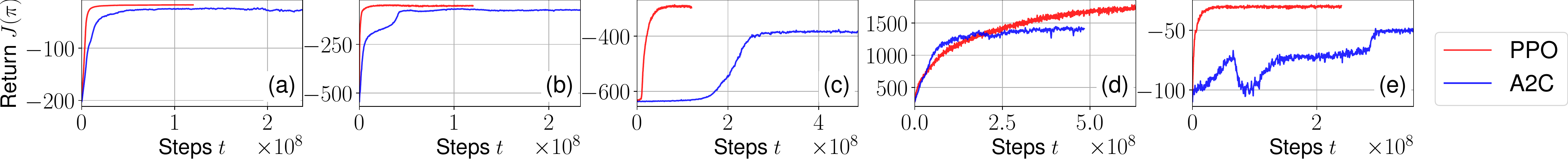}
    \caption{Training curves (mean episode return vs. time steps) of M3FPPO in red, compared to A2C in blue. (a) 2G; (b) Formation; (c) Beach; (d) Foraging; (e) Potential.}
    \label{fig:A2C}
\end{figure}

\begin{figure}
    \centering
    \includegraphics[width=0.47\linewidth]{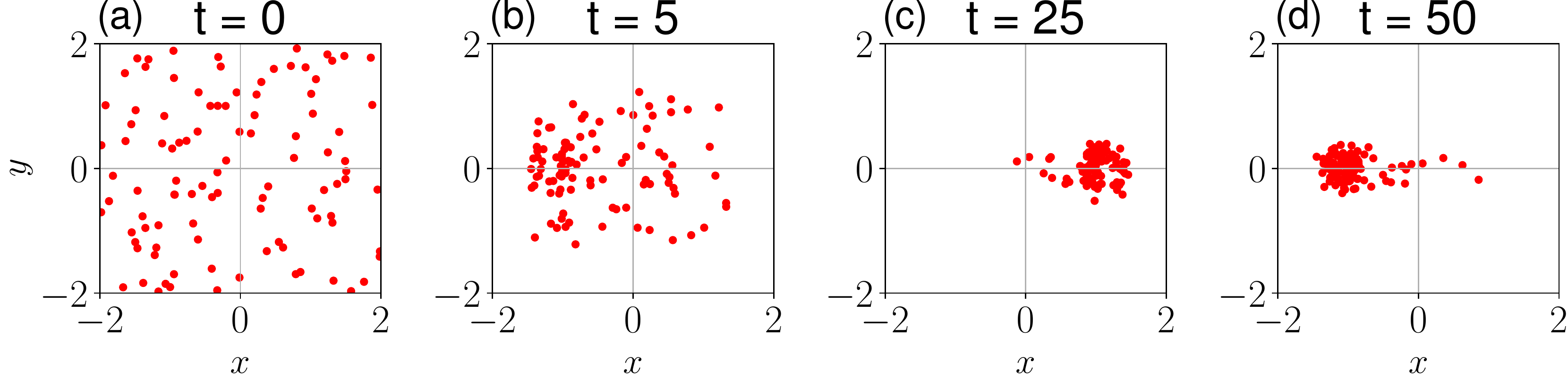}
    \includegraphics[width=0.47\linewidth]{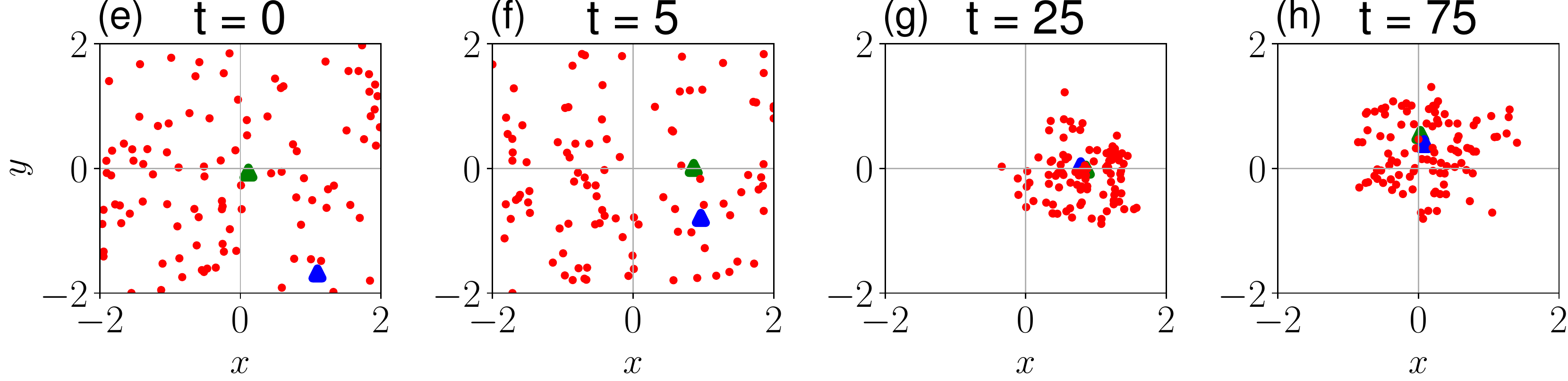}
    \includegraphics[width=0.47\linewidth]{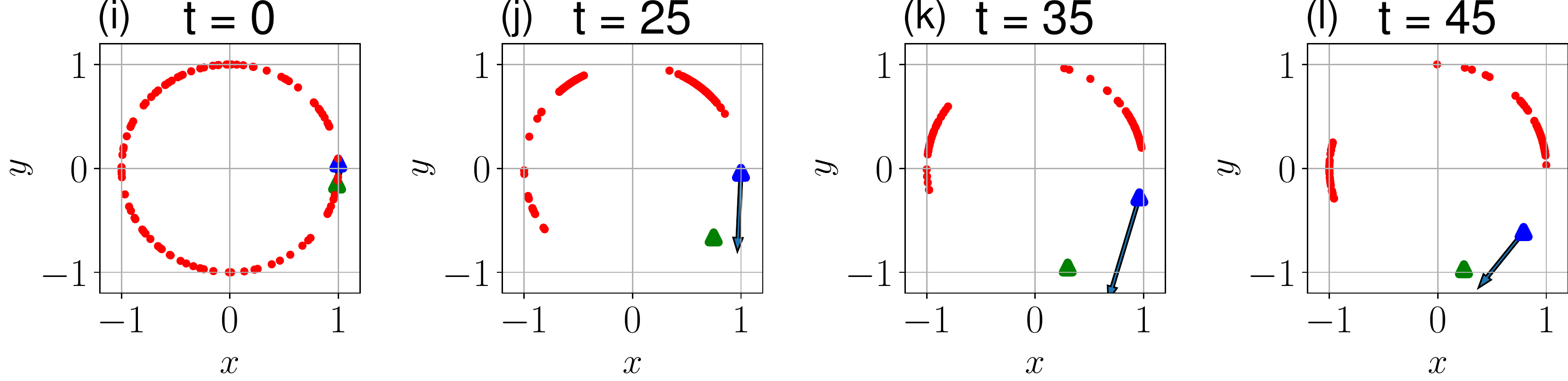}
    \caption{Qualitative visualization of learned M3FC behavior in the 2G (a-d), Formation (e-h) and Potential (i-l) problems. Red: minor agent; blue triangle: major agent; green triangle: major agent target. (i-l): As in (e-h), with arrow for potential gradient (not to scale).}
    \label{fig:qualitative1}
\end{figure}

\begin{figure}
    \centering
    \includegraphics[width=0.99\linewidth]{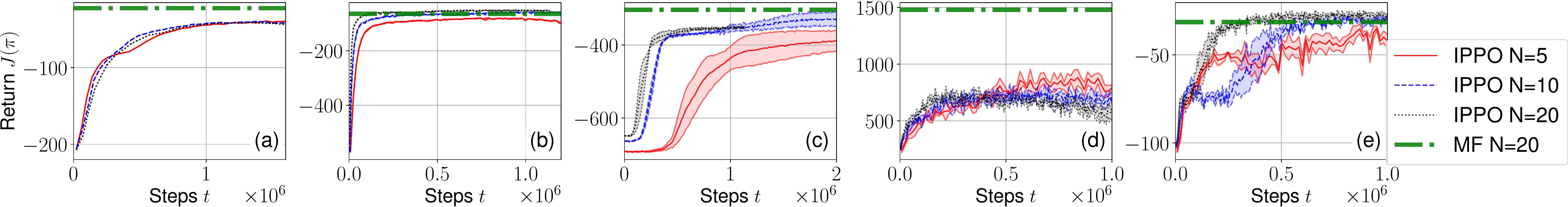}
    \caption{Training curves (mean episode return vs. time steps) of IPPO, trained on the systems with $N \in \{ 5, 10, 20 \}$. (a) 2G; (b) Formation; (c) Beach; (d) Foraging; (e) Potential.}
    \label{fig:ippo}
\end{figure}

\begin{figure}
    \centering
    \includegraphics[width=0.99\linewidth]{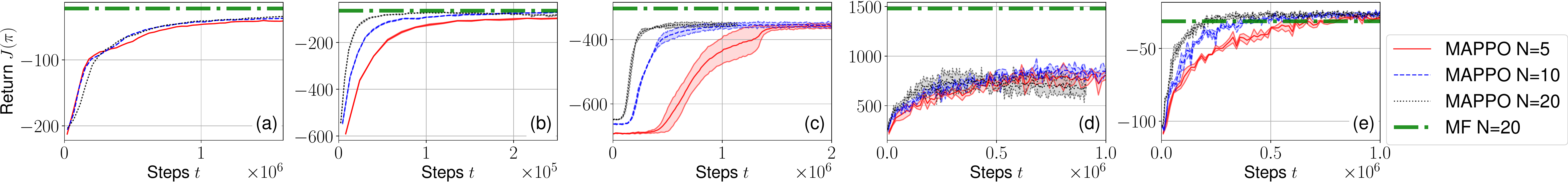}
    \caption{Training curves (mean episode return vs. time steps) of MAPPO, trained on the systems with $N \in \{ 5, 10, 20 \}$. (a) 2G; (b) Formation; (c) Beach; (d) Foraging; (e) Potential.}
    \label{fig:mappo}
\end{figure}

\subsection{Comparison to M3FA2C} \label{app:a2c}
In Figure~\ref{fig:A2C} we can see that vanilla M3FA2C typically performs worse than M3FPPO, getting stuck in worse local optima. Here, we used the same hyperparameters as in PPO. This validates our choice of PPO for M3FMARL.

\subsection{Qualitative results} \label{app:qual}
In Figure~\ref{fig:qualitative1}, M3FPPO successfully learns to form mixtures of Gaussians in 2G, and a Gaussian around a moving major agent that tracks its target in Formation. 
As expected in 2G, the two Gaussians at their sinusoidal peaks $t=25$ and $t=50$ are not perfectly tracked, in order to minimize the cost in following time steps, when the other Gaussian reappears. Finally, in Potential the minor agents succeed in pushing the major agent towards its target, while spreading on both sides of the major agent to be able to track any random movement of the target.

\subsection{Training M3FPPO, IPPO and MAPPO on smaller systems} \label{exp:extra2}
In Figure~\ref{fig:Ntrain} we verified the training of M3FPPO on small finite system. Comparing to Figures~\ref{fig:curvem3fc} and \ref{fig:Ncompare}, for M3FPPO we see little difference between training on a small finite-agent system versus training on a large system and applying the policy on the smaller system. For the chosen hyperparameters, the performance in the Potential problem depends on the initialization. However, M3FPPO compares especially favorably to IPPO in Beach and Foraging, even when directly training on the finite system. This shows that we can either (i) directly apply M3FPPO as a MARL algorithm to small systems, or (ii) train on a fixed system, and transfer the learned behavior to systems of almost arbitrary other sizes.

Analogously, in Figures~\ref{fig:ippo} and \ref{fig:mappo} we show the training results for around a day of IPPO and MAPPO for numbers of agents $N=5$, $N=10$ and $N=20$. As seen in the plot, the results for each number of agents is comparable to the analysis shown in the main text. In particular, transferring M3FPPO or comparing with Figure~\ref{fig:Ntrain}, we observe that M3FPPO continues to outperform or match the performance of IPPO and MAPPO, even in the setting with fewer agents.


\end{document}